\documentclass[letterpaper, 10 pt, conference]{ieeeconf}  % Comment this line out if you need a4paper
\IEEEoverridecommandlockouts                              % This command is only needed if 
\usepackage{soul}
\usepackage{url}
\usepackage[utf8]{inputenc}
\usepackage{graphicx}
\usepackage{amsmath}
\usepackage{booktabs}

\usepackage{amssymb,amsfonts}
\usepackage{textcomp}
\usepackage{subfigure}
\usepackage{float}
\usepackage{tabularx}
\usepackage{fancyhdr}
\usepackage{algorithm}    
\usepackage{algpseudocode}
\usepackage[np, autolanguage]{numprint}

\usepackage[colorlinks=true, allcolors=black]{hyperref}
\usepackage{epstopdf} %
\usepackage{color,xcolor}
\usepackage{xspace}
\usepackage{balance}
\newtheorem{Proposition}{Proposition}
\newtheorem{Definition}{Definition}
\newtheorem{Theorem}{Theorem}

\usepackage{tikz}
\usetikzlibrary{positioning}
\usetikzlibrary{shapes.geometric}
\usetikzlibrary{arrows.meta}
\usetikzlibrary{calc,angles}
\usetikzlibrary[shadings]
\tikzset{>={Stealth[width=3mm]}}
\usepackage{xcolor} % 用于定义和使用颜色
\definecolor{navy}{RGB}{0,0,128}
\definecolor{darkred}{RGB}{139,0,0}
\definecolor{darkgreen}{RGB}{0,100,0}
\definecolor{custompurple}{RGB}{128,0,128}
\definecolor{ppurple}{HTML}{9400D3}
\definecolor{deeporange}{RGB}{230,115,0}
\makeatletter
\newcommand{\gettikzxy}[3]{%
  \tikz@scan@one@point\pgfutil@firstofone#1\relax
  \edef#2{\the\pgf@x}%
  \edef#3{\the\pgf@y}%
}
\newcommand{\markRightAngle}[4][0.2cm]{
    \coordinate (tempa) at ($(#3)!#1!(#2)$);
    \coordinate (tempb) at ($(#3)!#1!(#4)$);
    \coordinate (tempc) at ($(tempa)!0.5!(tempb)$); %midpoint
    \draw[line width=1.5pt] (tempa) -- ($(#3)!2!(tempc)$) -- (tempb);
 }
\makeatother

\usepackage{setspace}

\renewcommand{\baselinestretch}{0.956} 

\title{\LARGE \bf Online Efficient Safety-Critical Control for Mobile Robots in Unknown Dynamic Multi-Obstacle Environments}

\author{Yu Zhang$^{1}$, Guangyao Tian$^{1}$, Long Wen$^{1}$, Xiangtong Yao$^{1}$, Liding Zhang$^{1}$, \\Zhenshan Bing$^{1}$, Wei He$^{2}$, and Alois Knoll$^{1}$
	\thanks{$^{1}$Authors from the Department of Informatics, Technical University of 
        Munich, Germany.
        {\tt\small zy.zhang@tum.de}}%
	\thanks{$^{2}$Authors from the Department of Intelligence Science and Technology, University of Science and Technology Beijing, Beijing 100083, China.}%
}
\begin{document}
\setstretch{0.92}
\maketitle
\thispagestyle{empty}
\pagestyle{empty}

%%%%%%%%%%%%%%%%%%%%%%%%%%%%%%%%%%%%%%%%%%%%%%%%%%%%%%%%%%%%%%%%%%%%%%%%%%%%%%%%
\begin{abstract}
This paper proposes a LiDAR-based goal-seeking and exploration framework, addressing the efficiency of online obstacle avoidance in unstructured environments populated with static and moving obstacles. This framework addresses two significant challenges associated with traditional dynamic control barrier functions (D-CBFs): their online construction and the diminished real-time performance caused by utilizing multiple D-CBFs. To tackle the first challenge, the framework's perception component begins with clustering point clouds via the DBSCAN algorithm, followed by encapsulating these clusters with the minimum bounding ellipses (MBEs) algorithm to create elliptical representations. By comparing the current state of MBEs with those stored from previous moments, the differentiation between static and dynamic obstacles is realized, and the Kalman filter is utilized to predict the movements of the latter. Such analysis facilitates the D-CBF's online construction for each MBE. To tackle the second challenge, we introduce buffer zones, generating Type-II D-CBFs online for each identified obstacle. Utilizing these buffer zones as activation areas substantially reduces the number of D-CBFs that need to be activated. Upon entering these buffer zones, the system prioritizes safety, autonomously navigating safe paths, and hence referred to as the exploration mode. Exiting these buffer zones triggers the system's transition to goal-seeking mode. We demonstrate that the system's states under this framework achieve safety and asymptotic stabilization. Experimental results in simulated and real-world environments have validated our framework's capability, allowing a LiDAR-equipped mobile robot to efficiently and safely reach the desired location within dynamic environments containing multiple obstacles. Videos are available from {\hypersetup{urlcolor=gray} \href{https://zyzhang4.wixsite.com/iros\_2024\_yuzhang}{https://zyzhang4.wixsite.com/iros\_2024\_yuzhang}}.
\end{abstract}

\section{Introduction}
Ensuring safety while accomplishing tasks in complex environments is a fundamental and challenging task that has been extensively researched within the robotics communities \cite{10015199}. The advent of control barrier functions (CBFs) has significantly advanced this field, offering a robust theoretical groundwork for safety through forward invariance \cite{ames2019control}. Two methodologies are usually incorporated with CBFs into a unified optimization problem to ensure safety while navigating to the desired goal configuration. One is model predictive control (MPC) \cite{9483029}, which can generate more complex and predictive behavior. The other is the control Lyapunov function (CLF) \cite{ames2016control}, known for conserving computational resources and its role in ensuring asymptotic stability. Among existing CBFs-based approaches, most works operate under the assumption that the obstacle locations are known in the state space and unsafe regions can be precisely and mathematically defined. However, such presupposes are impractical for real-world robotics applications due to the unknown dynamics of the environment. 

To enable robotic systems to identify unsafe areas in unfamiliar environments, cameras and LiDAR are utilized as the primary sensory tools. Compared to detecting unsafe areas through cameras, constructing CBFs based on point clouds in unstructured environments is better suited to adverse lighting and weather conditions, hence gaining favor in recent research \cite{9197481}. Initially, several works have concentrated on employing machine learning techniques to derive CBFs for static obstacles \cite{9341190,9392327,9981177}. In \cite{9341190}, support vector machines are utilized for distinguishing safe from unsafe areas via offline training. Following this, in \cite{9392327}, authors employ single-line LiDAR data for online CBFs estimation, adopting a neural network for offline learning of signed distance functions, and employing incremental learning with replay memory for updates. This approach, utilizing a five-layer neural network, necessitates a CPU update time of 0.197 seconds. Mirroring the approach of \cite{9392327}, \cite{9981177} leverages a deep neural network for direct online parameter adjustments, achieving a quicker update average of 0.15 seconds on a GPU with fewer obstacles. In addition to focusing solely on CBFs learning, various works have expanded their exploration of learning-based methodologies within more comprehensive frameworks. This includes end-to-end safe reinforcement learning \cite{cheng2019end}, end-to-end safe imitation learning \cite{9993193}, and integrating them into CLF-CBFs frameworks \cite{9676477}.  Although some efforts have further enhanced the robustness of these learning methods \cite{dawson2022safe} and extended them to stochastic systems \cite{9990576}, the real-world applicability of these approaches is constrained \cite{ames2019control}. This limitation arises from prolonged online update periods for more obstacles and the necessity for unchanging environmental conditions, which reduce their practical effectiveness. 

To fully consider dynamic obstacles in real environments, a novel ellipses-based MPC-CBF method that employs multi-step prediction has been introduced \cite{8768044,10160857}. In \cite{8768044}, researchers model dynamic obstacles with ellipsoidal shapes, defining precise boundaries for collision prediction. This foundation is further enhanced in \cite{10160857} by integrating Kalman filter estimation errors into the CBF framework, thus improving its prediction capabilities. However, when compared to learning-based approaches that generate an all-encompassing CBF via neural networks, the ellipses-based strategy faces certain challenges. Primarily, it necessitates the identification of the smallest axes that fully encompass unsafe areas, a process complicated by the need for higher-order discriminant analysis that adds to the computational burden. Furthermore, the requirement to formulate multiple CBFs to account for various obstacles significantly increases the time needed for solving multi-step prediction problems like those in MPC-CBF methodologies, affecting the feasibility of real-time applications. 

In this paper, we introduce a novel goal-seeking and exploration framework that addresses the challenges above, ensuring the safety and asymptotic stability of a LiDAR-equipped mobile robot in real-time navigating through unknown environments. This environment contains obstacles of unknown quantities and states, under the condition that at least one safe trajectory to the goal exists. 
The key contributions of our research are outlined as follows: 

%a) A refined local perception approach is introduced, facilitating the online differentiation of dynamic and static obstacles and enabling the creation of corresponding CBFs. b) User-defined buffer zones are applied to evolve CBFs into Type-II D-ZCBFs, effectively shrinking their activation area and ensuring that the number of obstacles does not affect the quantity of activated Type-II D-ZCBFs, thereby substantially reducing optimization times for scenarios involving multiple CBFs. c) A distinctive goal-seeking and exploration framework, merging the benefits of CLF and MPC, is introduced. Rigorous validations of the controller's safety and stability within this framework are provided, and its effectiveness is confirmed through rigorous experimental verification.

\begin{itemize}
    \item A refined local perception approach is introduced, facilitating the online differentiation of dynamic and static obstacles and creating corresponding D-CBFs.
    \item User-defined buffer zones are applied to evolve D-CBFs into Type-II D-CBFs, effectively shrinking their activation area and ensuring that the number of obstacles does not affect the quantity of activated Type-II D-CBFs, thereby substantially reducing optimization times for scenarios involving multiple D-CBFs.
    \item A distinctive goal-seeking and exploration framework, merging the benefits of CLF and MPC, is introduced. Rigorous validations of the controller's safety and stability within this framework are provided, and its effectiveness is confirmed through experimental verification.
\end{itemize}
\section{Problem Statement}
%\begin{figure}[!t]
%	\centering
%	\input{Figs/overview/tikz/rat_spine_1}
%	\caption{The rat robot with a soft actuated spine. The pink line shows the soft actuated spine. The diagram in the upper right briefly illustrates that the flexing spine can be considered as a segment of a circle with a central angle $\theta_s$.}
%	\label{fig:robot}
%\end{figure}

%In this section, we offer an overview of our rat robot, including its components and the unique characteristics of the soft actuated spine. Furthermore, we will provide a brief review of the motion influenced by spinal flexion, as investigated in our previous work 
We describe the kinematics of a nonholonomic wheeled mobile robot as control-affine systems:
\begin{equation} \label{model}
\begin{aligned}
& \dot{\boldsymbol{x}}=f(\boldsymbol{x})+g(\boldsymbol{x}) \boldsymbol{u} \\
& \boldsymbol{y}=z(\boldsymbol{x})
\end{aligned}
\end{equation}
where $\boldsymbol{x} \in \mathcal{X} \subseteq \mathbb{R}^n$ is the states of robot, $f(\boldsymbol{x}): \mathcal{X} \rightarrow \mathbb{R}^n$ and $g(\boldsymbol{x}): \mathcal{X} \rightarrow \mathbb{R}^{n \times m}$  exhibit local Lipschitz continuity within their respective domains, $\boldsymbol{y} \in \mathcal{Y} \subseteq \mathbb{R}^z$ is the output of system, and $u \in \mathcal{U}:=\left\{u \in \mathbb{R}^m \mid A u \leq b\right\}$ is the control inputs with $A \in \mathbb{R}^{m \times m}$ and $b \in \mathbb{R}^{m}$. The occupied areas of the moving robot are represented by $\mathcal{C}(\boldsymbol{z})$, which are denoted by the union of $N_{rob}$ circles, i.e., $\mathcal{C}(\boldsymbol{z})=\cup_{k=1}^{N_{rob}} \mathcal{C}_{k}(\boldsymbol{z})$. The robot is equipped with a LiDAR for perceiving obstacle information in the environment. Obstacles are categorized into two types: static and dynamic. For the $N_{sta}$ static obstacles, the occupied areas are denoted by $\mathcal{O}_{sta}$, i.e., $\mathcal{O}_{sta} = \cup_{i=1}^{N_{sta}} \mathcal{O}_{i}$, and these areas have no overlap or intersection, i.e., $\mathcal{O}_i \cap \mathcal{O}_j=\emptyset$ for any $0 \leq i, j \leq N_{obs}$. For the $N_{mov}$ dynamic obstacles at time $t$, the space they occupy is denoted as $\mathcal{O}_{mov}$, i.e., $\mathcal{O}_{mov} = \cup_{i=1}^{N_{mov}} \mathcal{O}_{i}$. At the same instant, these moving obstacles have no intersection as well. Additionally, the robot utilizes its LiDAR sensor to collect the environment data $\mathcal{P}_{t}$ at each moment, noted by $\mathcal{P}_{t}=\cup_{j=1}^{N_{surf}} \boldsymbol{p}_{j} \subset \mathcal{P}\left(\boldsymbol{x}_t\right) \cap \partial z\left(\mathcal{O}\right)$, where $\mathcal{P}\left(\boldsymbol{x}_t\right)$ denotes the data points corresponding to the limits of the LiDAR's sensing range $\mathcal{P}\left(\boldsymbol{x}_t\right)$, and $\partial z\left(\mathcal{O}\right)$ represents the data points that directly relate to the surface information of all obstacles. 

\emph{Problem Formulation:}
Consider an affine control robotic system \eqref{model}, operating within environments characterized by initially unknown unsafe sets $\mathcal{O} = \mathcal{O}_{sta} \cup \mathcal{O}_{mov}$. A goal-seeking and exploration framework with LiDAR measurements $\mathcal{P}_{t}$ and a feedback control policy $\boldsymbol{u}^{*}$ is proposed. The objectives of this framework are: a) To enable the online generation of CBFs based on $\mathcal{P}_{t}$. b) To significantly decrease the solution time compared to existing approaches, enhancing operational efficiency in dynamic settings. c) To navigate the robot towards a desired goal while avoiding unsafe sets, ensuring $\mathcal{C}(\boldsymbol{z}) \cap \mathcal{O}=\emptyset$.
%\begin{figure}[!t]
%	\centering
%	\input{Figs/overview/tikz/spine_1}
%	\caption{Schematic of the rat robot with spinal flexion. 
%	The blue-shaded region represents the robot's skeleton in its initial state, whereas the pink-shaded region depicts the robot's skeleton when flexing its spine.
%	$l_{HH}$ and $l_{FH}$ denote the lengths of the hind hip and fore hip, respectively. The dynamic stride lengths of the hind limb and fore limb are $l_h(t)$ and $l_f(t)$ respectively, which generate the robot's gait over time. $l_B$ is the length of the robot body. $l_S$ is the length of the spine. $R(t)$ signifies the time-varying flexing angle of the spine. $P_h$ and $P_f$ represent the footholds for the hind limb and fore limb during trotting, respectively. 
%	And the purple-shaded region connecting such two footholds presents the robot's support area. 
%	Moreover, to account for the influence of spinal flexion on the hind limb foothold, we introduce $l_{hx}(t)$ and $l_{hy}(t)$ to describe the alterations in the coordinates of the hind foothold due to spinal flexion.}
%	\label{fig:spine_flexion}
%\end{figure}

\section{Preliminary}
This section provides a detailed overview of the local safe path planner, which integrates MPC or CLF with D-CBFs as key components.
\subsection{System Safety and Control Barrier Functions}
To effectively consider both static and dynamic obstacles, D-CBFs are extensively utilized to ensure system safety. This methodology concurrently accounts for the system's state $\boldsymbol{x}$ and the obstacle states $\boldsymbol{x}_{obs}$. In a scenario where the operational space contains obstacles, the safe states, termed as $\mathcal{S}$, are defined as $\mathcal{S} = \left\{\boldsymbol{x}, \boldsymbol{x}_{obs} \in \mathcal{X} \mid h\left(z(\boldsymbol{x}, \boldsymbol{x}_{obs})\right) \geq 0\right\}$, aligning with the super-zero level-set of function $h(\cdot)$. The boundary of $\mathcal{S}$, pinpointed by the zero level-set, is described as $\partial \mathcal{S} = \left\{\boldsymbol{x}, \boldsymbol{x}_{obs} \in \mathcal{X} \mid h\left(z(\boldsymbol{x}, \boldsymbol{x}_{obs})\right) = 0\right\}$. Within this framework, system \eqref{model} is classified as \emph{safe} if it satisfies two conditions: firstly, the initial states $z(\boldsymbol{x}_{0}, \boldsymbol{x}_{obs}^{0})$ must be within $ \mathcal{S}$; secondly, for all $t \geq 0$, the state $z(\boldsymbol{x}, \boldsymbol{x}_{obs})$ must continuously remain within $\mathcal{S}$ \cite{ames2019control}.   
\begin{Definition} \cite{10160857}
Define the safety set $\mathcal{S} $ as the super-zero level set of the function $h\left(z(\boldsymbol{x}, \boldsymbol{x}_{obs})\right) : \mathcal{X} \rightarrow \mathbb{R}$. The function $h\left(z(\boldsymbol{x}, \boldsymbol{x}_{obs})\right)$ is considered as a candidate D-CBF if there exists a locally Lipschitz extended class $\mathcal{K}$ function $\alpha(\cdot)$, ensuring the maintenance of the following inequality for all $\boldsymbol{x}, \boldsymbol{x}_{obs} \in \mathcal{X}$:
  \begin{equation} \label{typeIIZCBF1}
      \sup _{\boldsymbol{u} \in \mathcal{U}}[\mathcal{L}_f h(\boldsymbol{x})+\mathcal{L}_g h(\boldsymbol{x}) \boldsymbol{u}+\underbrace{\frac{\partial h}{\partial \boldsymbol{x}_{obs}} \frac{\partial \boldsymbol{x}_{obs}}{\partial t}}_\varepsilon ]+\alpha(h) \geq 0 
  \end{equation}
where $\varepsilon$ represents the variability in the obstacle's position and shape affecting the safe sets. In cases where the obstacle is static, $\varepsilon$ assumes a value of zero, and \eqref{typeIIZCBF1} is called CBFs, indicating no change in its position or shape. The set of control inputs that conform to \eqref{typeIIZCBF1} is characterized as follows for all $\boldsymbol{x}, \boldsymbol{x}_{obs} \in \mathcal{X}$:
\begin{equation}
\mathcal{T}(\boldsymbol{x})=\left\{\boldsymbol{u} \in \mathcal{U} \mid \mathcal{L}_f h(\boldsymbol{x})+\mathcal{L}_g h(\boldsymbol{x}) u+ \varepsilon + \alpha(h) \geq 0\right\} . \nonumber
\end{equation} 
\end{Definition}
\subsection{Model Predictive Control with D-CBFs}
MPC is utilized to forecast future system behaviors and satisfy the specified tasks, including goal-seeking, path following, and tracking. The MPC-D-CBFs framework, structured to fulfill goal-seeking and safety tasks, is outlined as follows:
\begin{subequations} \label{mpc}
\begin{align}
J^* = \min _{\boldsymbol{x}_{0:N},\boldsymbol{u}_{0:N-1}} \sum_{k=0}^{N-1}& J\left(\boldsymbol{x}_{k}, \boldsymbol{u}_{k}\right)+J\left(\boldsymbol{x}_N, \boldsymbol{u}_N\right) \label{mpc_a} \\
\text{s.t.} \quad \boldsymbol{x}_{k+1} =f(\boldsymbol{x}_{k})&+g(\boldsymbol{x}_{k}) \boldsymbol{u}_{k}, k=0, \ldots, N-1 \label{mpc_b} \\
\boldsymbol{x}_k \in \mathcal{X}, \boldsymbol{u}_k &\in \mathcal{U}, k=0, \ldots, N-1 \label{mpc_c} \\
\boldsymbol{x}_{N} &\in \mathcal{X}_f \label{mpc_d} \\
\Delta h\left(\boldsymbol{x}_k, \boldsymbol{x}_{obs}^{k}\right) &\geq -\alpha(h), k=0, \ldots, N-1 \label{mpc_e}
\end{align}
\end{subequations}
where $N$ is the receding horizon, $J\left(\boldsymbol{x}_{k}, \boldsymbol{u}_{k}\right)$ represents the stage cost, $J\left(\boldsymbol{x}_N, \boldsymbol{u}_N\right)$ signifies the terminal cost, and \eqref{mpc_b} describes the discrete-time dynamics of the system. Concurrently, \eqref{mpc_e} formulates the discrete-time D-CBFs in terms of \eqref{typeIIZCBF1}, where $\Delta h\left(\boldsymbol{x}_k, \boldsymbol{x}_{obs}^{k}\right) = h\left(\boldsymbol{x}_{k+1}, \boldsymbol{x}_{obs}^{k+1}\right)-h\left(\boldsymbol{x}_k, \boldsymbol{x}_{obs}^{k}\right)$ represents the change in the barrier function over the discrete time steps. 
\subsection{Control Lyapunov Function with D-CBFs}
CLFs are utilized to accomplish tasks similar to those undertaken by MPC. Compared to MPC, CLFs effectively operate as a single-step prediction, thereby reducing the computational time required for getting control inputs. However, a notable limitation of CLFs is their difficulty in ensuring bounded control inputs, especially when the system is significantly distant from the goal.
\begin{Definition}\cite{ames2019control} (\emph{Asymptotically Control Lyapunov Function})
For stabilizing states of nonlinear system \eqref{model} to a goal $\boldsymbol{x}_{d}$, i.e., driving $ \boldsymbol{x}_{e} = \boldsymbol{x} - \boldsymbol{x}_d  \rightarrow 0$. If a continuous differentiable function $V: \mathbb{R}^{n} \rightarrow \mathbb{R}$ satisfies
\begin{equation} \label{CLF1}
c_1\|\boldsymbol{x}_{e}\|^2 \leq V(\boldsymbol{x}_{e}) \leq c_2\|\boldsymbol{x}_{e}\|^2
\end{equation}
\begin{equation} \label{CLF2}
\inf _{\boldsymbol{u} \in U}\left[\mathcal{L}_f V(\boldsymbol{x}_{e})+\mathcal{L}_g V(\boldsymbol{x}_{e}) \boldsymbol{u}+c_3 V(\boldsymbol{x}_{e})\right] \leq 0
\end{equation}
where $c_1>0$, $c_2>0$, and $c_3>0$ are constants for $\forall \boldsymbol{x} \in \mathbb{R}^n$, then $V$ is called the CLF, and the system \eqref{model} globally and exponentially achieves the ideal goal $\boldsymbol{x}_{d}$.
\end{Definition}
\begin{Proposition}\cite{ames2019control}
  (\emph{Sufficient condition for stability}) Upon selecting a candidate CLF $V(\boldsymbol{x}_{e})$ for the nonlinear system represented by \eqref{model}, the system can be rendered asymptotically stable by employing a control policy that satisfies $\pi(\mathbf{x}) \in\{\mathbf{u} \in U_\pi^v \mid \mathcal{L}_f V(\boldsymbol{x}_{e})+\mathcal{L}_g V(\boldsymbol{x}_{e}) \boldsymbol{u}+c_3 V(\boldsymbol{x}_{e})\leq 0 \}$.
  \label{preposition2}
\end{Proposition}

The aforementioned CLF and D-CBFs are integrated as a quadratic programming problem, facilitating the simultaneous achievement of stability and safety for the controlled system \eqref{model}. 
\begin{subequations} \label{my_equ6}
\begin{align}
& \boldsymbol{u}^*=\underset{(\boldsymbol{u}, \delta) \in \mathbb{R}^{m+1}}{\arg \min }\|\boldsymbol{u}\|_2^2+p \delta^2 \nonumber\\
& \text { s.t. } \inf _{\boldsymbol{u} \in U}\left[\mathcal{L}_f V(\boldsymbol{x}_{e})+\mathcal{L}_g V(\boldsymbol{x}_{e}) \boldsymbol{u}+c_3 V(\boldsymbol{x}_{e})\right] \leq \delta \label{my_equ7}\\
& \qquad \qquad\mathcal{L}_f h(\boldsymbol{x})+\mathcal{L}_g h(\boldsymbol{x}) u+ \varepsilon + \alpha(h) \geq 0,  \label{my_equ8}
\end{align}
\end{subequations}
where $\delta$ denotes a slack variable to obtain feasible solutions for the optimization problem, while still satisfying the D-CBF constraints strictly.
\subsection{Challenges}
Reviewing \eqref{mpc} and \eqref{my_equ6}, we identify two critical limitations in the existing approach: a) The challenge of constructing CBFs and D-CBFs online when the quantities and states of obstacles are unknown. b) The significant increase in solution time as the number of obstacles grows. The subsequent sections will address these two issues in detail.

\section{Safety-critical control based on Type-II D-ZCBFS}
This section presents a novel goal-seeking and exploration framework based on Type-II D-ZCBFs. This framework is structured into three key components: local perception, CBFs-activation, and safety-critical control.
\subsection{Type-II Zeroing Control Barrier Functions}
Introducing Type-II ZCBFs aims to decrease the activation length of CBFs, thereby reducing the need for activation during the whole control process.
\begin{figure}[htbp]
\centering
  \subfigure[]{
    \begin{minipage}[t]{0.47\linewidth}
        \centering
        \begin{tikzpicture}
        \node[fill opacity=0.1] at (0,0) {\includegraphics[width=\linewidth,trim=0 0 0 0,clip]{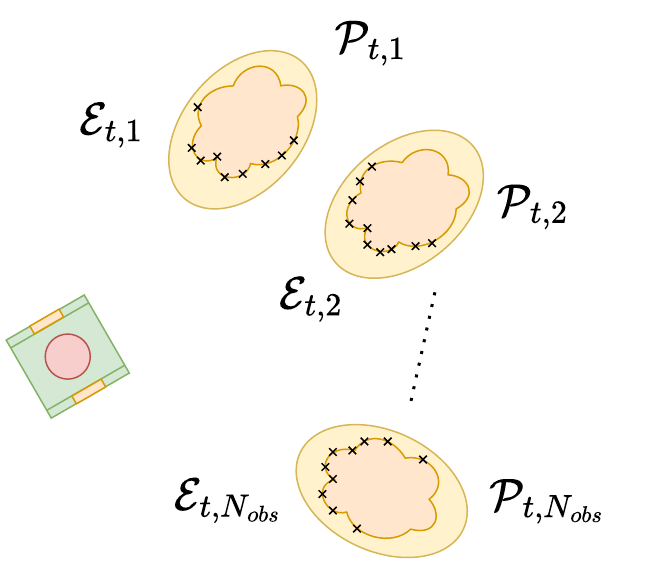}};
        \end{tikzpicture}
    \end{minipage}
   }%
    %\hfill
  % \hspace{0.001mm}
  \subfigure[]{
    \begin{minipage}[t]{0.47\linewidth}
        \centering
        \begin{tikzpicture}
        \node[fill opacity=0.1] at (0,0) {\includegraphics[width=\linewidth,trim=0 0 0 0,clip]{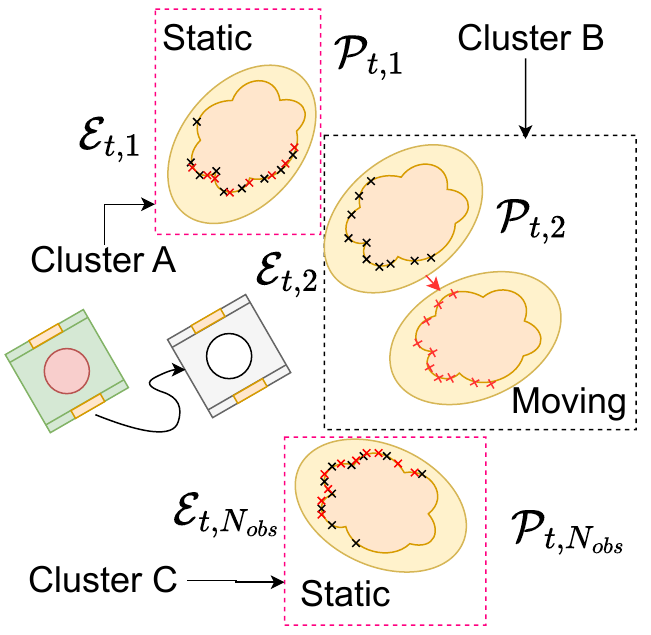}};
        \end{tikzpicture}
        \vspace{-1em}
    \end{minipage}
   }
  \caption{Visualization of the algorithm for clustering and classification. (a) The DBSCAN clustering algorithm is applied to point cloud data $\mathcal{P}_{t}$, subsequently resulting in $\mathcal{E}_{t,k}$ through the MBE algorithm.  (b) Clustering outcomes obtained by the Kuhn–Munkres algorithm.}
 \label{fig:improved_subfig}
 \end{figure}
\begin{Definition}
Let the safety set $\mathcal{S}$ be the zero super level set of $h\left(\boldsymbol{x}\right) : \mathcal{X} \rightarrow \mathbb{R}$. The function $h\left(\boldsymbol{x}\right)$ qualifies as a Type-II ZCBF concerning the set $\mathcal{B} = \left\{\boldsymbol{x} \in \mathcal{X} \subset \mathbb{R}^n \mid h\left(\boldsymbol{x}\right) \in \left[-b, a\right]\right\}$ if there exists continuous functions $\alpha^{*}(\cdot)$ satisfying $\left.\alpha^{*}\right|_{\mathbb{R}_{\geq 0}} \in$ class-$\mathcal{K}$ and $\alpha^{*}(h\left(\boldsymbol{x}\right)) \leq 0$ for all $h\left(\boldsymbol{x}\right) < 0$, such that the following inequality is upheld: 
  \begin{equation} \label{typeIIZCBF}
      \sup _{\boldsymbol{u} \in \mathcal{U}}\left[\mathcal{L}_f h(\boldsymbol{x})+\mathcal{L}_g h(\boldsymbol{x}) \boldsymbol{u}+\alpha(h(\boldsymbol{x}))\right] \geq 0, \forall \boldsymbol{x} \in \mathcal{B}
  \end{equation}
\end{Definition}

The set of control inputs adhering to \eqref{typeIIZCBF} is defined as follows for all $\boldsymbol{x} \in \mathcal{B}$:
\begin{equation}
\mathcal{T}^{*}(\boldsymbol{x})=\left\{\boldsymbol{u} \in \mathcal{U} \mid \mathcal{L}_f h(\boldsymbol{x})+\mathcal{L}_g h(\boldsymbol{x}) u+\alpha(h(\boldsymbol{x})) \geq 0\right\} . \nonumber
\end{equation}

\begin{Theorem}
Define the safety set $\mathcal{S}$ as the the zero super-level set of $h\left(\boldsymbol{x}\right) : \mathcal{X} \rightarrow \mathbb{R}$. Suppose each $h\left(\boldsymbol{x}\right)$ is a Type-II ZCBF with non-zero gradient $\nabla h\left(\boldsymbol{x}\right)$ at all $\boldsymbol{x} \in \partial \mathcal{S}$, considering the specified set $\mathcal{B}$ and function $\alpha^{*}(\cdot)$. (i) Given the existence of a local Lipschitz continuous mapping $\boldsymbol{u}: \boldsymbol{x} \in \mathcal{X} \mapsto \mathcal{T}^{*}(\boldsymbol{x})$, the corresponding closed-loop system \eqref{model} upholds safety relative to the set $\mathcal{S}$. (ii) If $\mathcal{B}$ is compact, $\alpha^{*}(\cdot)$ follows the chosen criteria, $\boldsymbol{x}$ remains bounded during the whole process, $\mathcal{D} = \mathcal{S} \cup \mathcal{B}$ forms a connected set, and no solution of the closed-loop system remains consistently within the set $\Psi:= \{\boldsymbol{x} \in \mathcal{D} \backslash \mathcal{S}: \dot{h}\left(\boldsymbol{x}\right)=0\}$, it implies the asymptotic stability of the set $\mathcal{S}$. 
\end{Theorem}

The proof of Definition 3 and Theorem 1 is provided in \cite{cortez2021robust}. This paper addresses dynamic obstacles by extending Type-II ZCBF to Type-II D-ZCBF and devising corresponding buffer zones, which are dynamically learned utilizing online LiDAR data.
\subsection{Local Perception}
This subsection outlines the methodology for distinguishing between static and dynamic obstacles utilizing online LiDAR data, encapsulating this information within MBEs. It also includes mechanisms for predicting future states of these MBEs, and for developing corresponding D-CBFs.
 \begin{figure}[htbp]
    \centering
        \begin{tikzpicture}
        \node[fill opacity=0.1] at (0,0) {\includegraphics[width=0.8\linewidth,trim=0 0 0 0,clip]{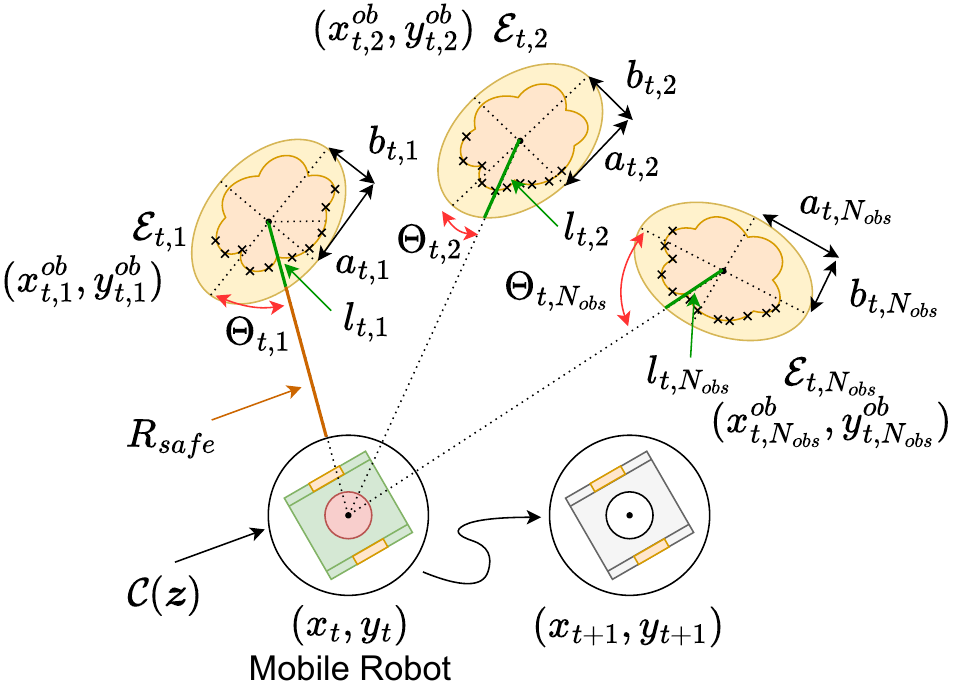}};
        \end{tikzpicture}
    \caption{A detailed explanation of the parameters for the mobile robot and MBEs at time $t$. $l_{t,k}$ is measured along the line that extends from the center $(x_{t,k}^{ob},y_{t,k}^{ob})$ of $\mathcal{E}_{t,k}$ to the center $(x_{t},y_{t})$ of the circular space $\mathcal{C}(\boldsymbol{z})$ occupied by the mobile robot. $R_{safe}$ is the minimum safe distance.}
    \label{fig:improved_subfig2}
    \vspace{-1em}
\end{figure}

Initially, we apply the DBSCAN method \cite{1ester1996density} to the point cloud data $\mathcal{P}_{t}$. This process involves setting a threshold of a minimum number of points, $N_{\min}$, and a maximum distance, $d_{p}$, between adjacent points. As a result, these point clouds are segmented into $N_{obs}$ distinct clusters at time $t$, represented as $\mathcal{P}_{t} = \sum_{k=1}^{N_{obs}} \mathcal{P}_{t,k}$. For each identified cluster $\mathcal{P}_{t,k}$, the MBE algorithm \cite{welzl2005smallest} is employed to generate $N_{obs}$ minimal bounding ellipses $\mathcal{E}_{t,k}$, $k=1,2,\ldots,N_{obs}$. Each ellipse $\mathcal{E}_{t,k}$ encapsulates the points within its corresponding spatial cluster $\mathcal{P}_{t,k}$, as shown in Fig. \ref{fig:improved_subfig}(a), and  
\begin{equation} \label{ellipses}
\begin{gathered}
f(\mathcal{E}_{t,k})=\frac{[(x-x_{t,k}^{ob}) \cos \theta_{t,k}^{ob}+(y-y_{t,k}^{ob}) \sin \theta_{t,k}^{ob}]^2}{a_{t, k}^2} \\
\quad+\frac{[(x-x_{t,k}^{ob}) \sin \theta_{t,k}^{ob}-(y-y_{t,k}^{ob}) \cos \theta_{t,k}^{ob}]^2}{b_{t,k}^2}-1 = 0,
\end{gathered}
\end{equation}
where \eqref{ellipses} refers to the set of points ${(x,y)}$  located on the surface of each ellipse. These ellipses are characterized by parameters including the central coordinate, the lengths of the semi-major and semi-minor axes, and the rotation angle, i.e., $\mathcal{E}_{t,k} = [x_{t,k}^{ob}, y_{t,k}^{ob}, a_{t, k}, b_{t, k}, \theta_{t,k}^{ob}]$.
 
To predict the future state changes of the estimated ellipses, we first store the state information of these ellipses from the previous moment $\mathcal{E}_{t}^{last} = [\mathcal{E}_{t,1}, \mathcal{E}_{t,2}, \ldots, \mathcal{E}_{t, N_{last}}]$, and conduct a comparison with the current state information, represented as $\mathcal{E}_{t}^{curr} = [\mathcal{E}_{t,1}, \mathcal{E}_{t,2}, \ldots, \mathcal{E}_{t, N_{curr}}]$. We then develop the affinity matrix $\mathcal{D} \in \mathbb{R}^{N_{last} \times N_{curr}}$ to record the distances $d_{i,j}$ between each pair of data centers in $\mathcal{E}_{t}^{last}$ and $\mathcal{E}_{t}^{curr}$, and the Kuhn–Munkres algorithm \cite{kuhn1955hungarian} is applied to link these datasets effectively. Additionally, any match where the distance $d_{i,j}$ exceeds the maximum threshold $d_{\max}$ is discarded. Ellipses in the new frame that do not find a match are assigned new labels. For ellipses within the same label, if the distance between their centers is less than $d_{\min}$, we classify them and the ellipses with new labels as static obstacles. Ellipses not meeting this criterion are considered dynamic obstacles, as shown in Fig. \ref{fig:improved_subfig}(b), and are subjected to state estimation via Kalman filtering \cite{welch1995introduction}. 
\begin{algorithm}
\caption{Online Local Perception Algorithm}
\begin{algorithmic}[1]
\small
\State \textbf{Inputs:} Online LiDAR data $\mathcal{P}_{t}$, robot's state $\boldsymbol{x}$, thresholds $d_{\max}$, $d_{\min}$, and previous MEBs $\mathcal{E}_{t}^{last}$.
\State Utilize DBSCAN to cluster $\mathcal{P}_{t}$, generating MBEs $\mathcal{E}_{t,k}$ for $k=1,2,\ldots, N_{curr}$.
\State Formulate $\mathcal{E}_{t}^{curr} = [\mathcal{E}_{t,1}, \mathcal{E}_{t,2}, \ldots, \mathcal{E}_{t, N_{curr}}]$.
\If{$\mathcal{E}_{t}^{last} \neq \emptyset$}
    \State Discriminate between dynamic and static MBEs utilizing the affinity matrix $\mathcal{D}$ and the Kuhn–Munkres algorithm.
    \State Apply the Kalman Filter to predict states for dynamic MBEs.
\Else
    \State Classify all MBEs as static.
\EndIf
\State Update $\mathcal{E}_{t}^{last} = \mathcal{E}_{t}^{curr}$.
\State Compute both D-CBFs for dynamic MBEs and CBFs for static MBEs utilizing \eqref{distance} and \eqref{Z-CBF-MBE}.
\State \textbf{Output:} Safety measures $h_{t,k}$ for $k=1,2,\ldots, N_{curr}$.
\end{algorithmic}
\end{algorithm}
For dynamic obstacles observed through MSE, we initially assume a constant motion velocity, designating the state variable as $\hat{\mathcal{E}}_{0,k}^{*} = [x_{0,k}^{ob}, y_{0,k}^{ob},\dot{x}_{0,k}^{ob}, \dot{y}_{0,k}^{ob},\Ddot{x}_{0,k}^{ob}, \Ddot{y}_{0,k}^{ob}, a_{0,k}, b_{0,k}, \theta_{0,k}]^{\top} \in \mathbb{R}^9$. Ultimately, by continuously adjusting and updating through Kalman filtering, we obtain the state $\hat{\mathcal{E}}_{t,k}^{*}$ at the current time $t$ and predict the state $\hat{\mathcal{E}}_{t+1,k}^{*}$ for the next instant. 

The standard Kalman filter process consists of two phases: prediction and update. For predictions extending beyond one step, the dynamic model can be applied successively for further predictions without proceeding to the update phase. We then construct the relative D-CBFs for each clustered static or moving obstacle $\mathcal{E}_{t,k}$. This process firstly includes determining the distance $l_{t,k}$ from each MBE's center to its surface. This distance is measured along the line that extends from the center of $\mathcal{E}_{t,k}$ to the center of the circular space $\mathcal{C}(\boldsymbol{z})$ occupied by the mobile robot. This geometric relationship is illustrated in Fig. \ref{fig:improved_subfig2}, and
\begin{equation} \label{distance}
l_{t,k}=\sqrt{\frac{a_{t,k}^2 b_{t,k}\left(1+\tan ^2 \Theta_{t,k}\right)}{b_{t,k}^2+a_{t,k}^2 \tan ^2 \Theta_{t,k}}}, k=0,1, \ldots, N_{obs}
\end{equation}
where $\Theta$ denotes the angle between the line connecting the obstacle's center to the mobile robot's center and the major axis of the ellipse, described as follows:
\begin{equation} \label{angle}
\begin{aligned}
\Theta_{t,k} = &\arctan 2((x_{t}-x_{t,k}^{ob}) \sin (\theta_{t,k}^{ob})+(y_{t,k}^{ob}-y_{t}) \cos (\theta_{t,k}^{ob}), \nonumber\\
& (x_{t,k}^{ob}-x_{t}) \cos (\theta_{t,k}^{ob})+(y_{t,k}-y_{t}) \sin (\theta_{t,k}^{ob})) 
\end{aligned}
\end{equation}
where $\arctan2(\cdot,\cdot)$ is a variant of the inverse tangent function. This formula incorporates the angle $\theta_{t,k}^{ob}$, as well as the positions of the mobile robot's center $(x_{t}, y_{t})$ and the ellipse's center $(x_{t,k}^{ob},y_{t,k}^{ob})$ at time $t$.

Leveraging \eqref{distance} along with state variables $\hat{\mathcal{E}}_{t,k}^{*}$ and $\hat{\mathcal{E}}_{t+1,k}^{*}$, we formulate the D-CBFs for each MBE as follows:
\begin{equation} \label{Z-CBF-MBE}
\begin{aligned}
h_{t,k}\left(z(\boldsymbol{x}, \boldsymbol{x}_{obs})\right) & = \left\|z(\boldsymbol{x}_{t})-z(\boldsymbol{x}_{t,k}^{ob})\right\|_2-l_{t,k} \\
& -R_{safe}-r, \quad k=0,1, \ldots, N_{obs}
\end{aligned}
\end{equation}
where $z(\boldsymbol{x}_{t}) = [x_{t}, y_{t}] \in \mathbb{R}^{2}$, $z(\boldsymbol{x}_{t,k}^{ob}) = [x_{t,k}^{ob},y_{t,k}^{ob}] \in \mathbb{R}^{2}$, $R_{safe}$ signifies the minimum safety distance that should be maintained between the obstacle and the mobile robot, while $r$ represents the radius of the circular area occupied by the mobile robot. Note that our local perception approach for \eqref{Z-CBF-MBE} is suitable for both discrete and continuous system representations, which is the foundation for the subsequent goal-seeking and exploration strategies.

\subsection{Goal-Seeking and Exploration Framework}
Observing \eqref{Z-CBF-MBE}, it becomes evident that as $N_{obs}$ increases, the optimization time for the traditional framework, as detailed in \eqref{mpc} and \eqref{my_equ6}, also extends. This subsection focuses on this issue. 

\begin{figure}[htbp]
\centering
  \begin{minipage}[t]{0.47\linewidth}
    \centering
    \begin{tikzpicture}
      \node[fill opacity=0.1] at (0,0) {\includegraphics[width=0.95\linewidth,trim=0 0 0 0,clip]{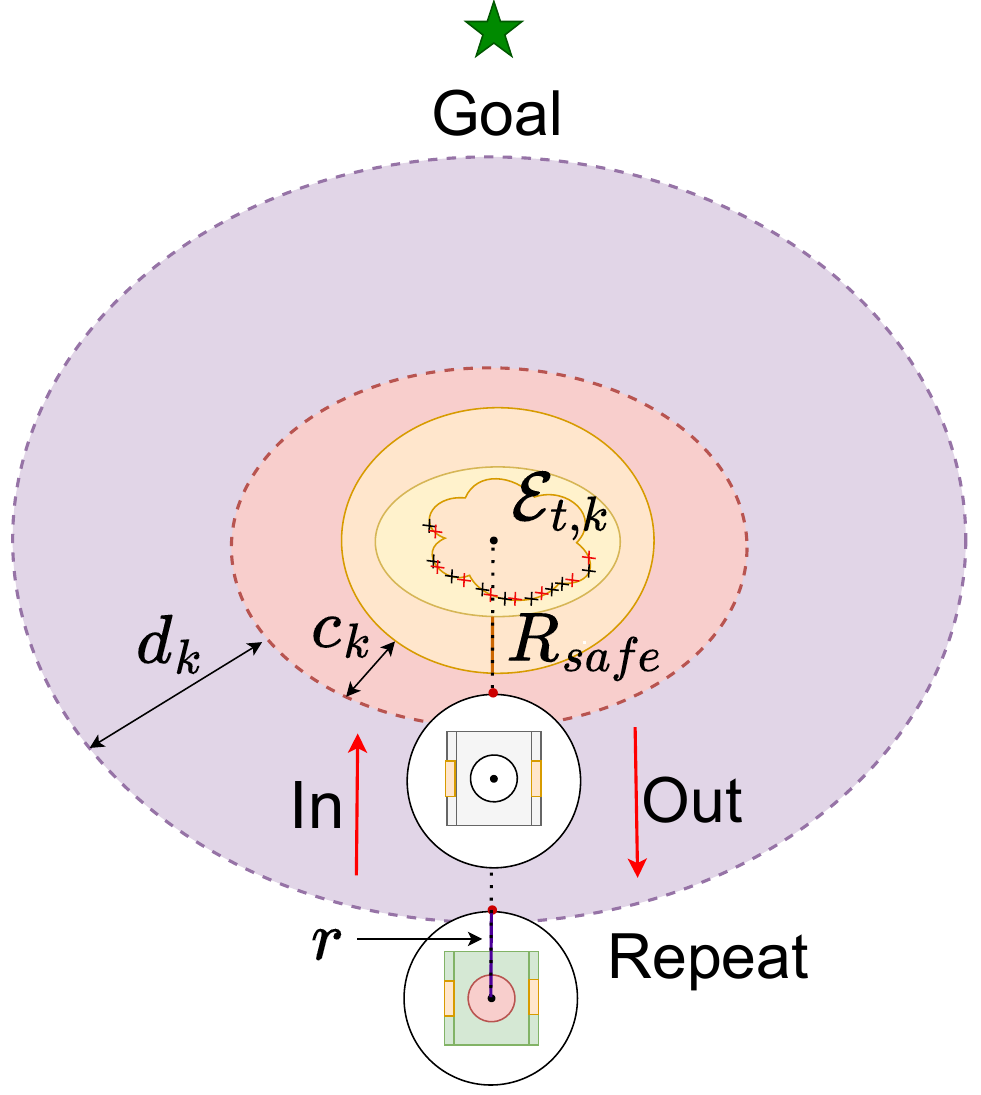}};
      \node[anchor=north west, font=\small] at (current bounding box.north west) {(a)}; 
    \end{tikzpicture}
  \end{minipage}%
  \hfill
  \begin{minipage}[t]{0.47\linewidth}
    \centering
    \begin{tikzpicture}
      \node[fill opacity=0.1] at (0,0) {\includegraphics[width=0.95\linewidth,trim=0 0 0 0,clip]{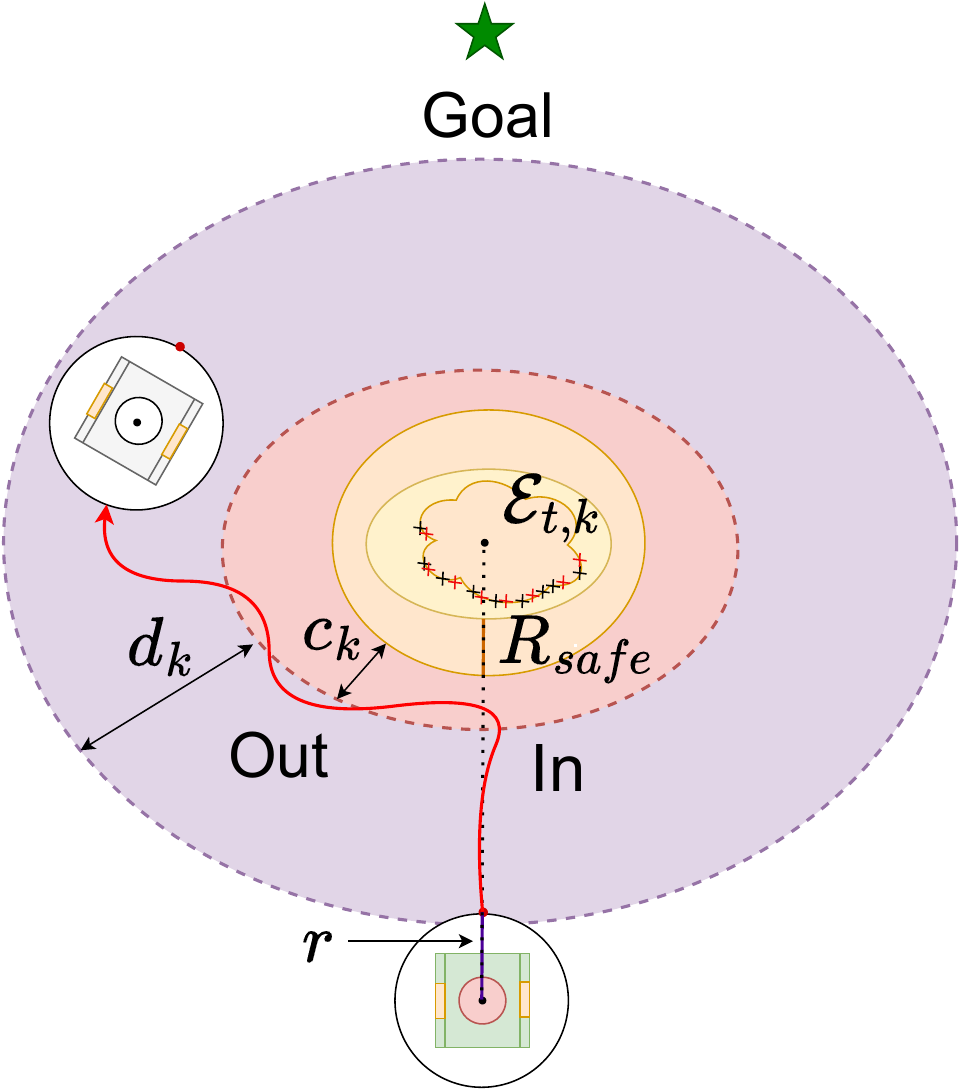}};
      \node[anchor=north west, font=\small] at (current bounding box.north west) {(b)}; 
    \end{tikzpicture}
  \end{minipage}
  \caption{The goal-seeking and exploration framework, with and without the backup safety controller for specific scenarios: (a) Without the backup safety controller, the mobile robot oscillates at the boundaries of two buffer zones, ensuring safety but unable to reach the goal fast. (b) With the backup safety controller, the mobile robot can alter its path through multi-step prediction, facilitating goal attainment.}
 \label{fig:improved_subfig1}
 \vspace{-0.4em}
\end{figure}

Our primary strategy involves creating a buffer zone for each MSE, activating the corresponding Type-II D-ZCBF only when the system's state enters this buffer zone, instead of activating all. We initially analyze scenarios where buffer zones do not overlap, which can be readily extended to cases with overlapping buffer zones. First, we define the safety set and the buffer set for $h_{t,k}(\cdot)$, $k=0,1, \ldots, N_{obs}$:
\begin{subequations}
\begin{align}
& \hat{\mathcal{S}}_{t,k}=\{z(\boldsymbol{x}_{t},\boldsymbol{x}_{t,k}^{ob}) \in \mathcal{X}: \hat{h}_{t,k} = h_{t,k}(z)-c_k \geq 0\}, \label{safe}\\
& \mathcal{B}_{t,k}=\{z(\boldsymbol{x}_{t},\boldsymbol{x}_{t,k}^{ob}) \in \mathcal{X}: \hat{h}_{t,k} \in\left[-c_k, d_k\right]\},\label{buffer}
\end{align}   
\end{subequations}
where $c_k$ and $d_k$ are user-defined parameters to specify the relative buffer set. If all $\hat{\mathcal{S}}_{t,k}$ are satisfied, then safety set $\mathcal{S}$ can also be assured due to $\mathcal{S} \subseteq \cap_{k=1}^{N_{obs}} \hat{\mathcal{S}}_{t,k}$.

The controller for maintaining safety with one MSE is introduced as: 
\begin{equation} \label{controller}
\boldsymbol{u}_{k}=(1-\rho_t(\hat{h}_{t,k})) \boldsymbol{u}_{s}^{k}+\rho_t(\hat{h}_{t,k}) \boldsymbol{u}_{nom}
\end{equation}
where $\boldsymbol{u}_{nom}$ represents the control input derived from the MPC or CLF method, exclusively focused on the goal-seeking task and not incorporating CBF considerations. The function $\rho_t(\hat{h}_{t,k})$  is defined as follows:
\begin{equation} \label{PPP}
\rho_t(\hat{h}_{t, k})= \begin{cases}1 & \text { if } \hat{h}_{t, k}>d_k \\ p(\hat{h}_{t, k}) & \text { if } \hat{h}_{t, k} \in\left[0, d_k\right] \\ 0 & \text { if } \hat{h}_{t, k}<0\end{cases}
\end{equation}
where $p(\cdot)$ is a continuous function that increases from $p(0) = 0$ to $p(d_{k}) = 1$, analogous to the error-shifting function described in \cite{10226609}. 

The safety-controller $\boldsymbol{u}_{s}^{k}$ is chosen as:
\begin{algorithm}
\caption{Goal-Seeking and Exploration Algorithm}
\small{
\begin{algorithmic}[1]
\State \textbf{Inputs:} Perception data $h_{t,k}$ for $k=1,2,\ldots, N_{curr}$, robot state $\boldsymbol{x}$, safety threshold $d_{bf}$, buffer zone boundaries $c_{k}$ and $d_{k}$, initialization $m=0$, $i=0$, and previous state $z_{last}^{j}$ setting $\hat{h}_{t,k} = 0$.
\For{$k = 1$ to $N_{curr}$}
    \State Construct buffer zones for $h_{t,k}$, obtaining $\hat{h}_{t,k}$ via \eqref{safe} and \eqref{buffer}.
    \If{$-c_{k} \leq \hat{h}_{t,k} \leq d_{k}$} 
        \State Set $h_{t,m} = \hat{h}_{t,k}$; increment $m$.
    \EndIf
    \If{$\hat{h}_{t,k} = 0$} 
        \State Log current position $z_{cur}^{i} = z(\boldsymbol{x})$; increment $i$.
    \EndIf
\EndFor
\If{any $\left\|z_{\text{cur}}^{i}-z_{last}^{j}\right\| < d_{bf}$} 
    \State Apply backup safety control $\boldsymbol{u}_{s}^{k} = \boldsymbol{u}_{b,s}^{k}$; compute $\boldsymbol{u}^{*}$ with \eqref{all-controller} utilizing activated $h_{t,m}$.
\Else
    \State Apply standard safety control $\boldsymbol{u}_{s}^{k}$; compute $\boldsymbol{u}^{*}$ with \eqref{all-controller} utilizing activated $h_{t,m}$.
\EndIf
\State \textbf{Output:} Safety control input $\boldsymbol{u}^{*}$.
\end{algorithmic}}
\end{algorithm}
\begin{subequations}\label{safe-controller}
\begin{align}
\boldsymbol{u}_{s}^{k} &= \underset{\boldsymbol{u}}{\arg \min} \left\|\boldsymbol{u} - \boldsymbol{u}_{nom}^{k}\right\|^{2}_{2}  \nonumber\\
\text{s.t.} \quad A \boldsymbol{u} &\leq b  \label{safe-controllerba}\\
\mathcal{L}_f \hat{h}_{t, k}(z) &+ \mathcal{L}_g \hat{h}_{t, k}(z) \boldsymbol{u} \nonumber \\
+ \varepsilon &+ \alpha(\hat{h}_{t, k}(z)) \geq 0, \ \forall z(\boldsymbol{x}_{t},\boldsymbol{x}_{t,k}^{ob}) \in \mathcal{B}_{t,k} \label{safe-controllerb}
\end{align}
\end{subequations} 
where \eqref{safe-controllerb} is pivotal in ensuring the system's safety. Given the convexity of $\mathcal{U}$, both $\boldsymbol{u}_{s}^{k}$ and $\boldsymbol{u}_{nom}^{k}$ residing within $\mathcal{U}$ are locally Lipschitz continuous. Therefore, $\boldsymbol{u}_{k}$ not only belongs to $\mathcal{U}$ but also retains this local Lipschitz continuity\cite{cortez2021robust}.
\begin{Theorem}
    When the system's distance from the MEB satisfies the safety criteria $l_{t,k}+R_{saf}+r$ and exceeds $c_{k}+d_{k}$, it is in the goal-seeking mode and ultimately converges to the desired position.
\end{Theorem}
\begin{proof}
    From \eqref{controller}, it is evident that when the aforementioned conditions are met, $\rho(z(\boldsymbol{x}) )=1$, resulting in the sole activation of $\boldsymbol{u}_{nom}^{k}$. At this point, the system ignores all obstacles, dedicating its efforts to achieving the goal-seeking objective. 
\end{proof}
\begin{Theorem}
    When the distance between the system and the MEB meets the safety criteria $l_{t,k}+R_{saf}+r$ and exceeds it within the range $[c_{k},c_{k}+d_{k}]$, the system will not remain in this region indefinitely.
\end{Theorem}
\begin{proof}
    Upon entering the aforementioned region, if the system ceases movement within the buffer zone $[c_{k},c_{k}+d_{k}]$, it necessitates fulfilling the condition $\rho_t(\hat{h}_{t,k})(\boldsymbol{u}_{s}^{k}-\boldsymbol{u}_{nom}^{k}) =  \boldsymbol{u}_{s}^{k}$ according to \eqref{controller}, implying that the coefficients of $\boldsymbol{u}_{s}^{k}$ and $\boldsymbol{u}_{nom}^{k}$ are directly inverse and proportional to each other. However,  Given that \eqref{safe-controller} aims to minimally alter $\boldsymbol{u}_{nom}^{k}$ for safety and with $\boldsymbol{u}_{s}^{k}$ limited by \eqref{safe-controllerba}, the inverse relationship of these is not satisfied. As a result, there are two possible outcomes: a) if partial safety is not enough for system safety, the system moves into the buffer zone $[0,c_{k})$, or b) the system switches back to goal-seeking mode. Therefore, the states are not lingering within the buffer zone $[c_{k},c_{k}+d_{k}]$.  
\end{proof}
 
However, relying solely on a one-step prediction like \eqref{safe-controller} can lead to the system's state oscillating in and out of the buffer zones $[-c_{k},0)$ and $[0,d_{k}]$ with altering its trajectory small as shown in Fig. \ref{fig:improved_subfig1}(a). To mitigate this, we designed a backup controller $\boldsymbol{u}_{b,s}^{k}$ based on \eqref{mpc},  adapting the discrete form of \eqref{Z-CBF-MBE} into a format similar to \eqref{mpc_e}. This multi-step prediction modifies trajectories that would otherwise remain small changed. When the system's position is at $\hat{h}_{k,i} = 0$ and the distance of the last state and the current state is less than $d_{bf}$, i.e., $\left\|z(\boldsymbol{x}_{cur})-z(\boldsymbol{x}_{last})  \right\| \leq d_{bf}$,  $\boldsymbol{u}_{b,s}^{k}$ is utilized as illustrated in Fig. \ref{fig:improved_subfig1}(b). This highlights the advantage of \eqref{Z-CBF-MBE} being adaptable to both continuous and discrete controllers.

Ultimately, the controller designed to ensure safety across all MSEs is presented as follows:
\begin{equation} \label{all-controller}
\boldsymbol{u}^{*}=(1-\bar{\rho}_t) \bar{\boldsymbol{u}}_{s}+\bar{\rho}_t \boldsymbol{u}_{nom}
\end{equation}
where $\bar{\rho}_t = \rho_t(\hat{h}_{t,k})$ if $z(\boldsymbol{x}_{t}) \in \mathcal{B}_{t,k}$, $k = 1, 2, \ldots, N_{obs}$, and $\bar{\rho}_t = 1$ otherwise. $\bar{\boldsymbol{u}}_{s} = \boldsymbol{u}_{s}^{k}$ if $z(\boldsymbol{x}_{t}) \in \mathcal{B}_{t,k}$, $\left\|z(\boldsymbol{x}_{cur})-z(\boldsymbol{x}_{last})  \right\| > d_{bf}$, $k = 1, 2, \ldots, N_{obs}$, and $\bar{\boldsymbol{u}}_{s} = \boldsymbol{u}_{b,s}^{k}$ if $z(\boldsymbol{x}_{t}) \in \mathcal{B}_{t,k}$, $\left\|z(\boldsymbol{x}_{cur})-z(\boldsymbol{x}_{last})  \right\| \leq d_{bf}$, $k = 1, 2, \ldots, N_{obs}$, and $\bar{\boldsymbol{u}}_{s} = 0$ otherwise. 
\begin{Theorem}
    Consider $N_{obs}$ Type-II ZCBFs with discrete and continuous form $\hat{h}_{t,k} : \mathcal{X} \rightarrow \mathbb{R}$ for each $k=1,2,\ldots,N_{obs} \in \mathcal{N}_{obs}$ at time instant $t$. For any distinct $k_{i}$, $k_{j} \in \mathcal{N}_{obs}$, the sets $\mathcal{B}_{t,k_{i}}$ and $\mathcal{B}_{t,k_{j}}$ are disjoint $(\mathcal{B}_{t,k_{i}} \cap \mathcal{B}_{t,k_{j}} = \emptyset)$. With the locally Lipschitz continuous control input $\boldsymbol{u}_{nom} \in \mathcal{U}$, if each Type-II ZCBF $\hat{h}_{t,k}$ is associated with a safety control input $\boldsymbol{u}_{s}^{k} \in \mathcal{U}$ and a backup safety controller $\boldsymbol{u}_{b,s}^{k}\in \mathcal{U}$, then the proposed control input $\boldsymbol{u}^{*}$, as in \eqref{all-controller} for system \eqref{model}, ensures that  

    (i) If \( z(\boldsymbol{x}_0) \) belongs to the intersection of all \( \mathcal{\hat{S}}_{t,k} \) for \( k \in \mathcal{N}_{obs} \), the system \eqref{model} is safe during the whole process.
    
    (ii) If the system's distance from the MEB satisfies and surpasses the safety threshold $l_{t,k}+R_{saf}+r$  within the interval $[0,c_{k})$, it will not persist an extended period.
    
    (iii) If the set \( \mathcal{U} \) is convex, then \( \boldsymbol{u}^* \) is within \( \mathcal{U} \) for all \( z(\boldsymbol{x}) \in \mathcal{X} \).
\end{Theorem}

\begin{proof}
 Since for any distinct $k_{i}$, $k_{j} \in \mathcal{N}_{obs}$, the intersection $\mathcal{B}_{t,k_{i}}$ and $\mathcal{B}_{t,k_{j}}$ is empty, there are only two possibilities for each buffer zone $\mathcal{B}_{t,k}$:  either there exists a unique $k$ for which $z(\boldsymbol{x}) \in \mathcal{B}_{t,k}$ or $z(\boldsymbol{x}) \notin \mathcal{B}_{t,k}$. Firstly, in the case where $z(\boldsymbol{x})$ does not belong to any given  $\mathcal{\hat{B}}_{k,i}$, i.e., each $\bar{\rho}_{t} = 1$, $\boldsymbol{u}^{*} = \boldsymbol{u}_{nom}$ is well-defined and locally Lipschitz continuous to complete the task. 
 
 Secondly, in the scenario where $\hat{h}_{t,k}$ falls within the range $[-c_{k}, 0)$, $\bar{\rho}_{t}$ is set to zero and $\dot{\hat{h}}_{t,k} = 0 $ becomes active solely at $\hat{h}_{t,k}=0$, resulting in no solution persisting consistently within the set $\Psi_{k,i} := \{\boldsymbol{x} \in (\mathcal{\hat{S}}_{t,k} \cup \mathcal{B}_{t,k} \setminus \mathcal{\hat{S}}_{t,k}): \dot{\hat{h}}_{t,k} = 0\}$.  According to Theorem 1, the safe set $\mathcal{\hat{S}}_{t,k}$ achieves asymptotic stability, thereby continuously ensuring real safety of $\mathcal{S} \subseteq \cap_{k=1}^{N_{obs}} \hat{\mathcal{S}}_{t,k}$. Similarly since $\boldsymbol{u}_{s}^{k}$ and $\boldsymbol{u}_{b,s}^{k}$ is also well-defined and locally Lipschitz continuous, $\boldsymbol{u}^{*}$ is well-defined.
 
 Finally, we analyze the case where $-c_{k} \leq \hat{h}_{t,k} \leq d_{k}$. Given that for any \( z(\boldsymbol{x}) \in \mathcal{X} \), $\boldsymbol{u}^{*}$ is formulated as a convex combination of $\bar{\boldsymbol{u}}_{s}$ and $\boldsymbol{u}_{nom}$, both of which belong to the convex set $\mathcal{U}$, it follows that $\boldsymbol{u}^{*}$ is locally Lipschitz continuous and is also an element of $\mathcal{U}$.  
\end{proof}

Theorem 4 confirms the system's safety when governed by the controller outlined in \eqref{all-controller}. Given that the primary objective within the buffer zone $\mathcal{B}$ is to activate the corresponding Type II D-ZCBF and ensure system safety, this segment focuses on exploring safe paths rather than goal-seeking. Thus, the operational state of the system within this buffer zone is described as being in exploration mode.
\begin{table}[ht]
\caption{Comparison of average solving time[Unit: \textcolor{red}{second}]}
\centering
\resizebox{0.48\textwidth}{!}{%
\begin{tabular}{p{2.7cm}|c|c|c|c}
 \toprule
    &1 obs &5 obs &10 obs &20 obs\\
 \hline
 MPC-Base  &0.0141   &0.0368   &0.1101 & 0.314\\
 MPC-All  &0.0743 & 0.1098 &0.197 & 0.295\\
 NMPC-CBF  &0.138 & 0.1432 &0.134 & 0.138\\
 MPC-Zone$^{*}$  &0.0465 & 0.06 & 0.0495 & 0.058\\
 MPC-CBF-Zone$^{*}$  &\textbf{0.00947} & \textbf{0.013} & \textbf{0.0112} & \textbf{0.0127}\\
 \bottomrule
\end{tabular}%
}
\label{tab: shape_control}
\end{table}
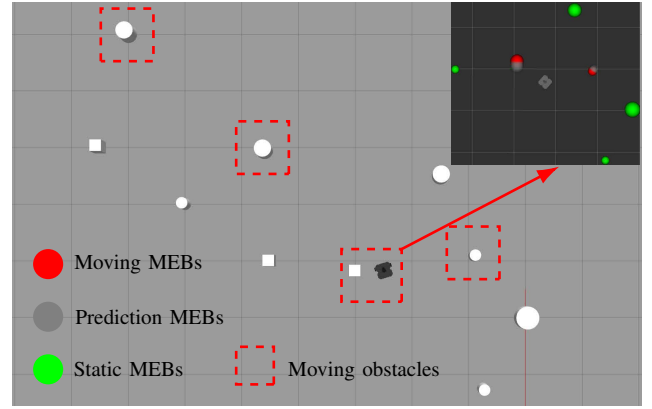
\begin{figure}[!t]
  \centering
  \begin{tikzpicture}
    \begin{scope}[xshift=-3.5cm, yshift=0cm ]
        \node[above right] (fig63) at (-2,0){\includegraphics[ height=0.62\linewidth]{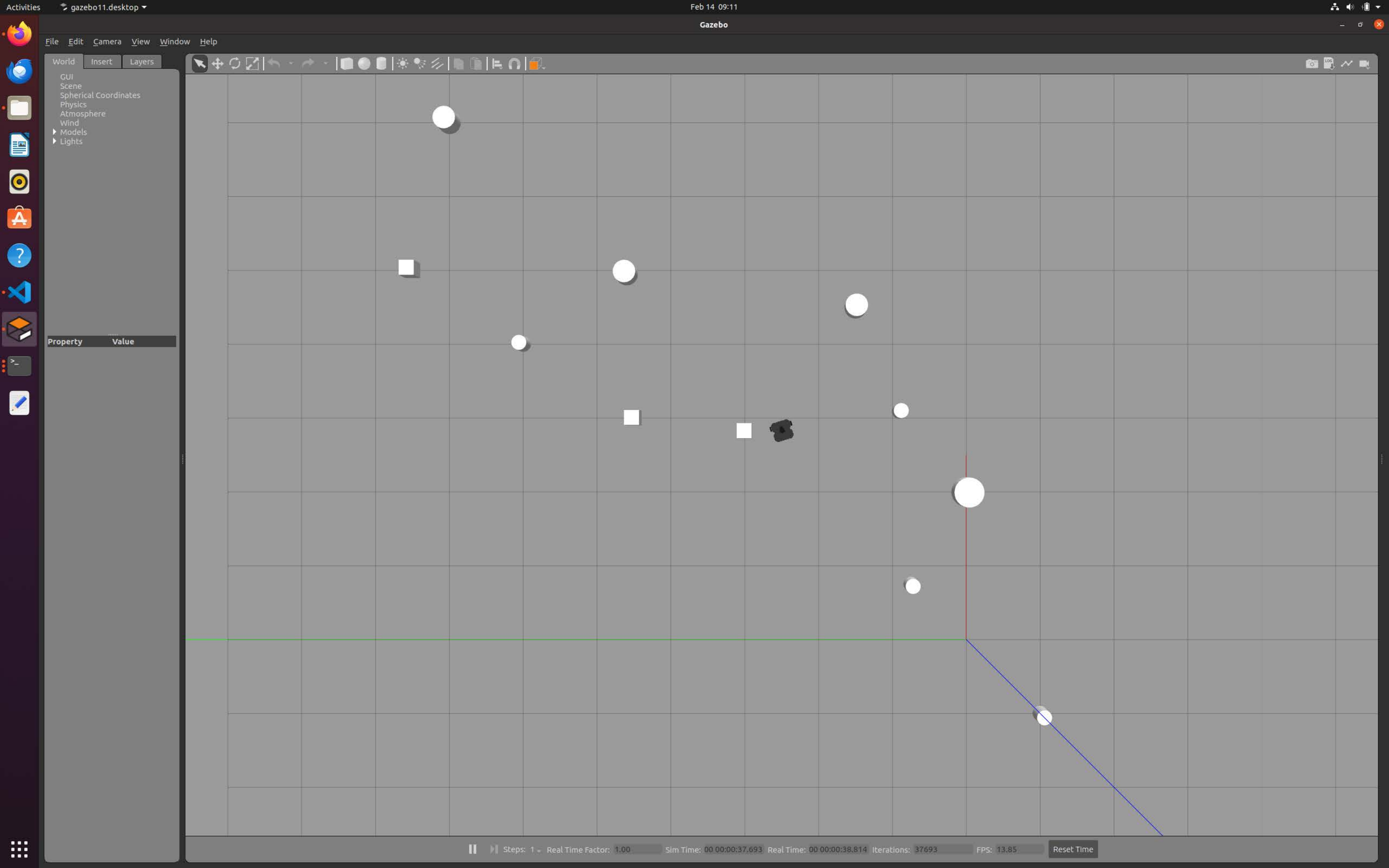}};
         \draw[red, dashed, line width=1pt] (-0.7,4.7) rectangle (-0,5.4);
         \draw[red, dashed, line width=1pt] (1.8,3.2) rectangle (1.1,3.9);
         \draw[red, dashed, line width=1pt] (3.3,1.5) rectangle (2.5,2.2);
         \draw[red, dashed, line width=1pt] (4.6,1.8) rectangle (3.9,2.5);
         \fill[red] (-1.4,2) circle (0.2);
         \fill[gray] (-1.4,1.3) circle (0.2);
         \fill[green] (-1.4,0.6) circle (0.2);
         \node[inner sep=1pt, font=\footnotesize] at (-0.2,2) {Moving MEBs};
         \node[inner sep=1pt, font=\footnotesize] at (-0.05,2-0.7) {Prediction MEBs};
         \node[inner sep=1pt, font=\footnotesize] at (-0.33,2-1.4) {Static MEBs};
         \draw[red, dashed, line width=1pt] (1.6,0.4) rectangle (1.1,0.9);
         \node[inner sep=1pt, font=\footnotesize] at (2.8,2-1.42) {Moving obstacles};
        %\node[rectangle, rounded corners=1pt,minimum width=2.1cm,minimum height=0.3cm,draw=gray,fill=white!10](box) at ($(fig63.north east)+(-1.23,-0.36)$) {};
    \end{scope}
    \draw[-{Latex[length=3mm, width=2mm]}, red, line width=1pt] (-0.2,2.2) -- (1.9,3.3);
    \begin{scope}[xshift=1cm, yshift=0cm]
        \node[above right] (fig64) at (-0.66,3.2){\includegraphics[ height=0.25\linewidth]{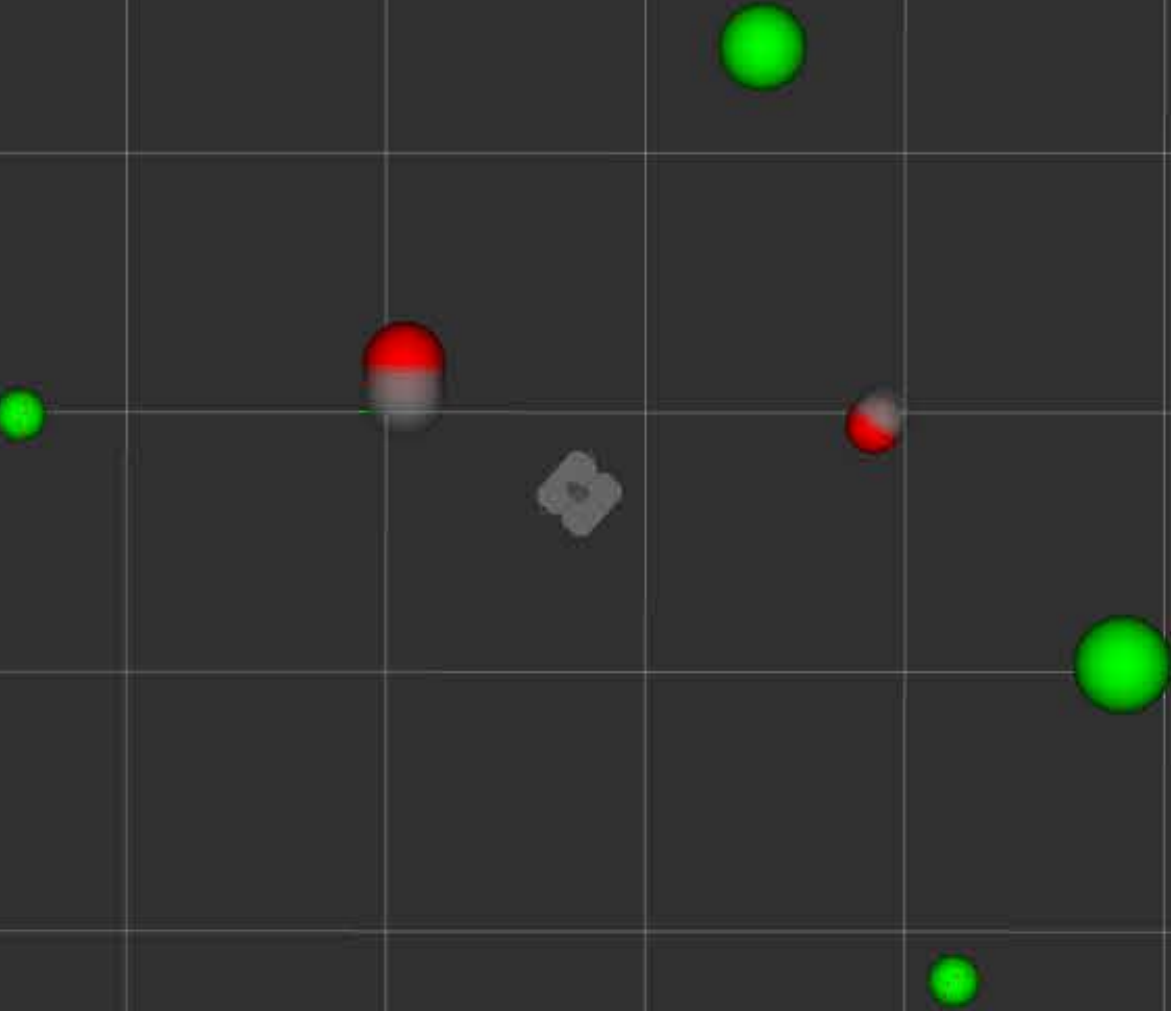}};
        %\node[rectangle, rounded corners=1pt,minimum width=2.1cm,minimum height=0.3cm,draw=gray,fill=white!10](box) at ($(fig64.north east)+(-1.23,-0.36)$) {};
    \end{scope}
  \end{tikzpicture}
  \caption{Simulation setups and perception results for moving obstacle scenarios.}
  \label{fig6}
  \vspace{-1em}
\end{figure}
\begin{Theorem}
When the goal lies beyond the buffer and unsafe zones, and at least one viable path to the destination exists, our framework, utilizing \eqref{all-controller}, ultimately directs the system towards achieving the goal.
\end{Theorem}
\begin{proof}
    Drawing from Theorems 3 and 4, the system \eqref{model} will not perpetually remain within the buffer zones $\mathcal{B}$. Consequently, the time it takes for the system to reach the goal is dependent on $\boldsymbol{u}_{nom}$, along with the stability of either the MPC or CLF strategies. Fortunately, the asymptotic stability of both MPC and CLF has been extensively researched and verified \cite{faulwasser2015nonlinear,grimm2005model,ames2019control}. As $t$ approaches infinity, the system asymptotically stabilizes at the goal. 
\end{proof}

Finally, we discuss the influence of buffer zone dimensions on system control. Keeping $d_{k}$  small is essential to limit the effect of $\bar{\boldsymbol{u}}_{s}$ on $\boldsymbol{u}_{nom}$, yet setting $d_{k}$ too small may increase the controller's sensitivity to Lidar's measurement noise. Conversely, a larger $c_{k}$ expands the attraction region for $\hat{\mathcal{S}}_{t,k}$, enhancing the system's capability to manage greater disturbances. However, this expansion necessitates a higher level of control authority. Thus, users can select appropriate $c_{k}$ and $d_{k}$ values based on their specific requirements. The proposed framework in this paper can be seamlessly adapted to situations involving overlapping buffer zones, by incorporating appropriate \eqref{safe-controllerb} for these intersections. In this overlapping situation, our method also offers computational efficiency by not requiring all D-ZCBFs. The overall process is depicted in Algorithms 1 and 2.

%\begin{figure}[!t]
%	\centering
%	\input{Figs/model/tikz/body}
%	\setlength{\abovecaptionskip}{-1pt}
%	\caption{Frontal view of the unbalanced robot. $RH$ and $LF$ are two limbs within a stance phase. The red arrow is the direction of gravity. The y-axis and z-axis are based on the local coordinate system of the robot. And $\theta_{\text{roll}}$ is the roll angle of the robot and can indicate the degree of tilt of the robot.}
%	\label{fig:body}
%	\vspace{-0.5em}
%\end{figure}

\section{Simulations and Experiments}
For performance evaluation, simulations featuring TurtleBot3 (TB3) robots were carried out \cite{9981177}, considering scenarios with both dynamic and static obstacles. For static obstacle scenarios, we benchmarked five methods and accounted for buffer overlap scenarios, setting $c_{k} = 0.1$m and $d_{k} = 0.2$m. MPC-Base is the baseline with full prior knowledge of obstacles. MPC-All is the latest ellipses-based approach \cite{10160857}, and NMPC-CBF is the advanced online learning-based method \cite{9981177}. Our method, MPC-Zone, solely utilizes the backup controller $\boldsymbol{u}_{b,s}^{k}$ for safety, whereas MPC-CBF-Zone combines this with the safety controller $\boldsymbol{u}_{s}^{k}$, as shown in \eqref{all-controller}.

Figure. \ref{fig5} presents the goal-directed trajectories generated by different approaches across varying obstacle densities. Specifically, Fig. \ref{fig5}(a)-(b) demonstrate that all evaluated methods effectively navigate scenarios with less than five obstacles. However, Table \ref{tab: shape_control} highlights a limitation in the NMPC-CBF method, where the slow online learning rate of the neural networks prevents timely controller adjustments, resulting in arc-shaped trajectories. Conversely, Table \ref{tab: change_accuracy_rising} showcases the superiority of the MPC-Base method in achieving the quickest goal reaching, attributed to its comprehensive pre-acquisition of obstacle data. When the number of obstacles reaches ten, in Fig. \ref{fig5}(c), MPC-All struggles with delays, taking 0.197s for online obstacle detection and optimization, while NMPC-CBF, aiming for a swift update speed of 0.14s, lacks the comprehensive data needed to accurately model complex environments via a CBF, resulting in obstacle collisions. When the obstacle count increases to twenty, Fig. \ref{fig5}(d) shows that only our method (yellow-brown and green line) successfully
\begin{figure}[!t]
  \centering
  \begin{tikzpicture}
    % Legend - Line with black color
    \draw[black, line width=1.5pt] (-10+0.7,0) -- (-9.5+0.7,0) node[right, text=black, font=\footnotesize]{MPC-Base};
    % Legend - Line with navy color
    \draw[navy, line width=1.5pt] (-7.8+0.7,-0) -- (-7.2+0.7,-0) node[right, text=black, font=\footnotesize]{MPC-All};
    % Legend - Line with navy color
    \draw[ppurple, line width=1.5pt] (-5.6+0.7,-0) -- (-5.1+0.7,-0) node[right, text=black, font=\footnotesize]{NMPC-CBF};
    % Legend - Line with dashed line
    \draw[orange, dashed, line width=1pt] (-3.2+0.7,-0) -- (-2.7+0.7,-0) node[right, text=black]{$c_{k}$};
    % Legend - Line with dotted line
    \draw[custompurple, dashed, line width=1pt] (-10+0.4,-0.5) -- (-9.5+0.4,-0.5) node[right, text=black]{$d_{k}$};
    % Legend - Line with darkred color
    \draw[darkred, line width=1.5pt] (-8.8+0.4,-0.5) -- (-8.3+0.4,-0.5) node[right, text=black, font=\footnotesize]{MPC-Zone$^{*}$ (Type-II)};
    % Legend - Line with darkgreen color
    \draw[darkgreen, line width=1.5pt] (-5.4+0.4,-0.5) -- (-4.9+0.4,-0.5) node[right, text=black, font=\footnotesize]{MPC-CBF-Zone$^{*}$ (Type-II)};
\end{tikzpicture}

  \hspace{-1.5em}
  \begin{tikzpicture}
    \begin{scope}[xshift=-3.5cm, yshift=0cm ]
        \node[above right] (fig63) at (0,0){\includegraphics[ height=0.334\linewidth]{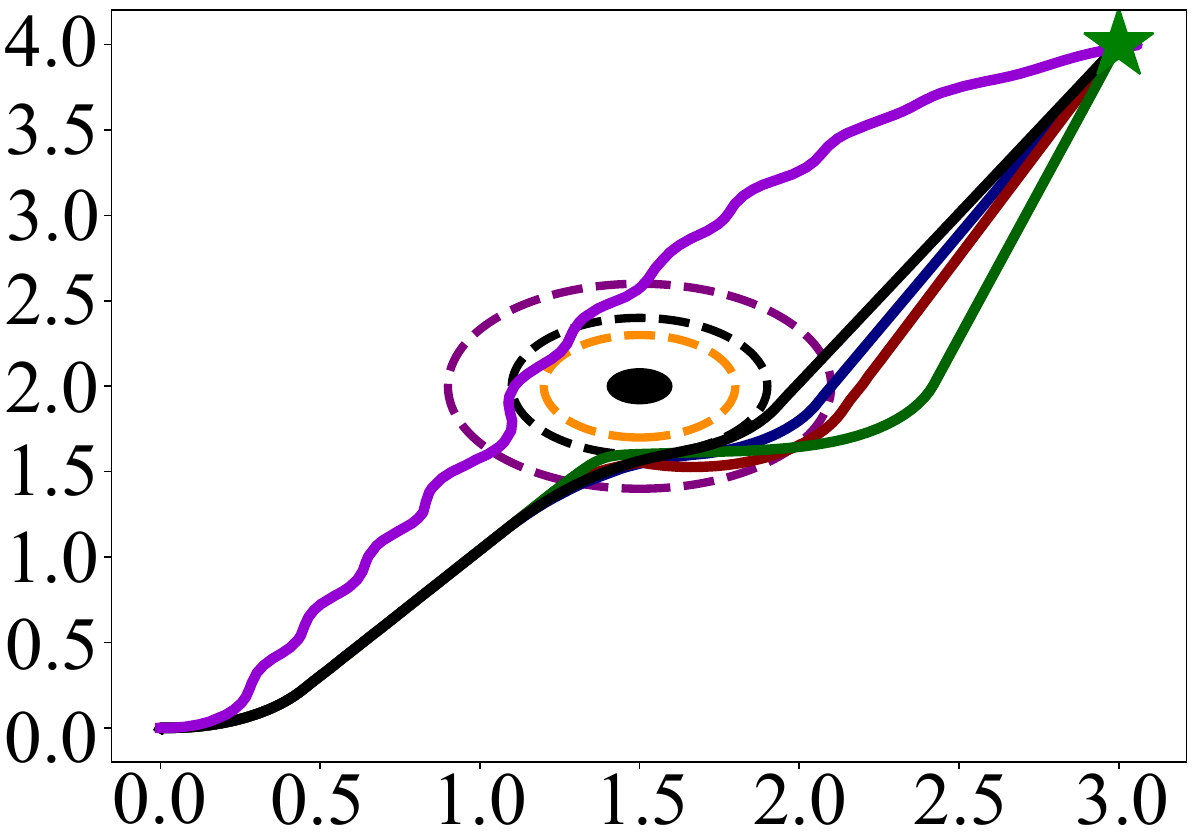}};
        \node at ($(fig63.south)+(0,-0.35)$) {\normalsize (a) one obstacle};
        %\node[rectangle, rounded corners=1pt,minimum width=2.1cm,minimum height=0.3cm,draw=gray,fill=white!10](box) at ($(fig63.north east)+(-1.23,-0.36)$) {};
            \node[rotate=90] at ($(fig63.west)+(0,0)$) {\scriptsize y positio [m]};
            \node at ($(fig63.south)+(-0.08,-0.02)$) {\scriptsize x position [m]};
    \end{scope}
    \draw[-{Latex[length=3mm, width=2mm]}, red, line width=1pt] (0.4,2.9) -- (-1,2.7) 
    node[pos=1, left, text=black, font=\footnotesize]{Goal};
    \node[inner sep=1pt, font=\footnotesize] at (0.59,2.63) {$p_1$};
    \begin{scope}[xshift=1cm, yshift=0cm]
        \node[above right] (fig64) at (0,0){\includegraphics[ height=0.334\linewidth]{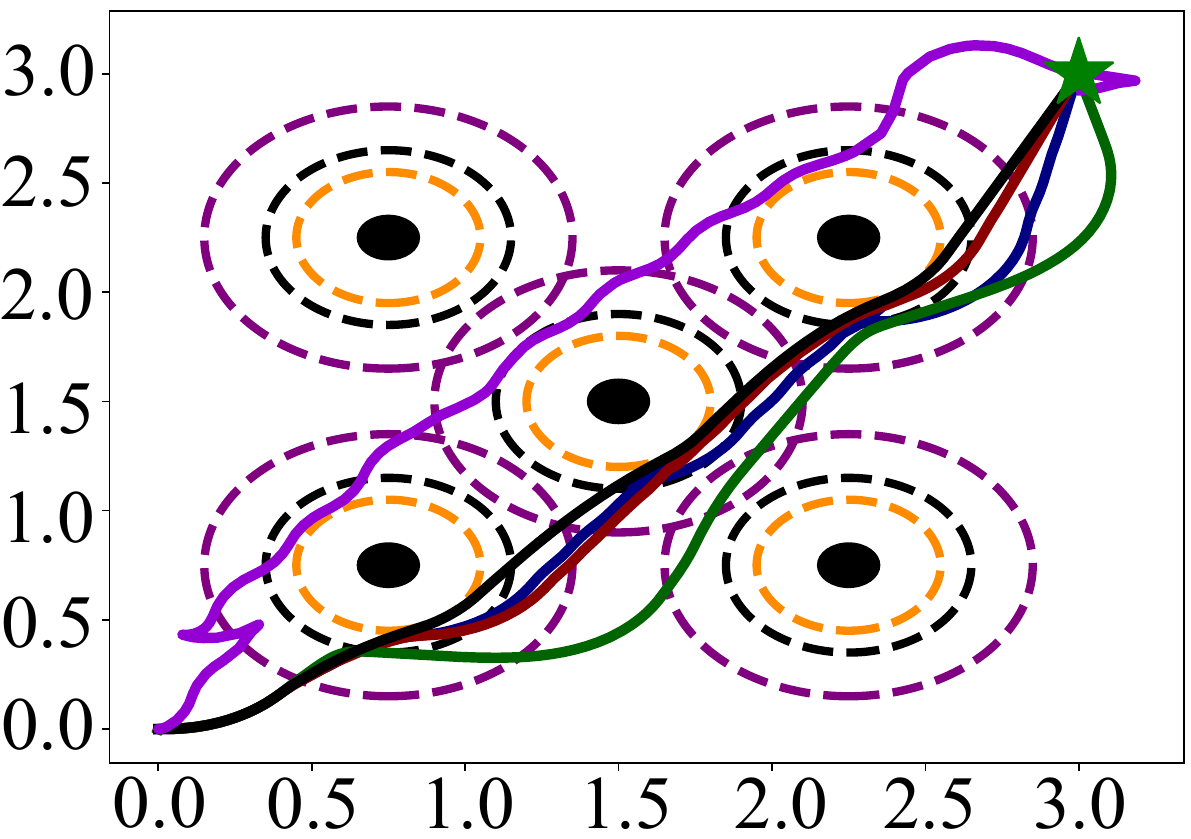}};
        \node at ($(fig64.south)+(0,-0.35)$) {\normalsize (b) five obstacles};
        %\node[rectangle, rounded corners=1pt,minimum width=2.1cm,minimum height=0.3cm,draw=gray,fill=white!10](box) at ($(fig64.north east)+(-1.23,-0.36)$) {};
            \node[rotate=90] at ($(fig64.west)+(-0.0,0)$) {\scriptsize y position [m]};
            \node at ($(fig64.south)+(0,-0.02)$) {\scriptsize x position [m]};
    \end{scope}
    \node[inner sep=1pt, font=\footnotesize] at (5.05,2.53) {$p_1$};
    
    \begin{scope}[xshift=-3.5cm, yshift=-3.8cm ]
        \node[above right] (fig63) at (0,0){\includegraphics[ height=0.334\linewidth]{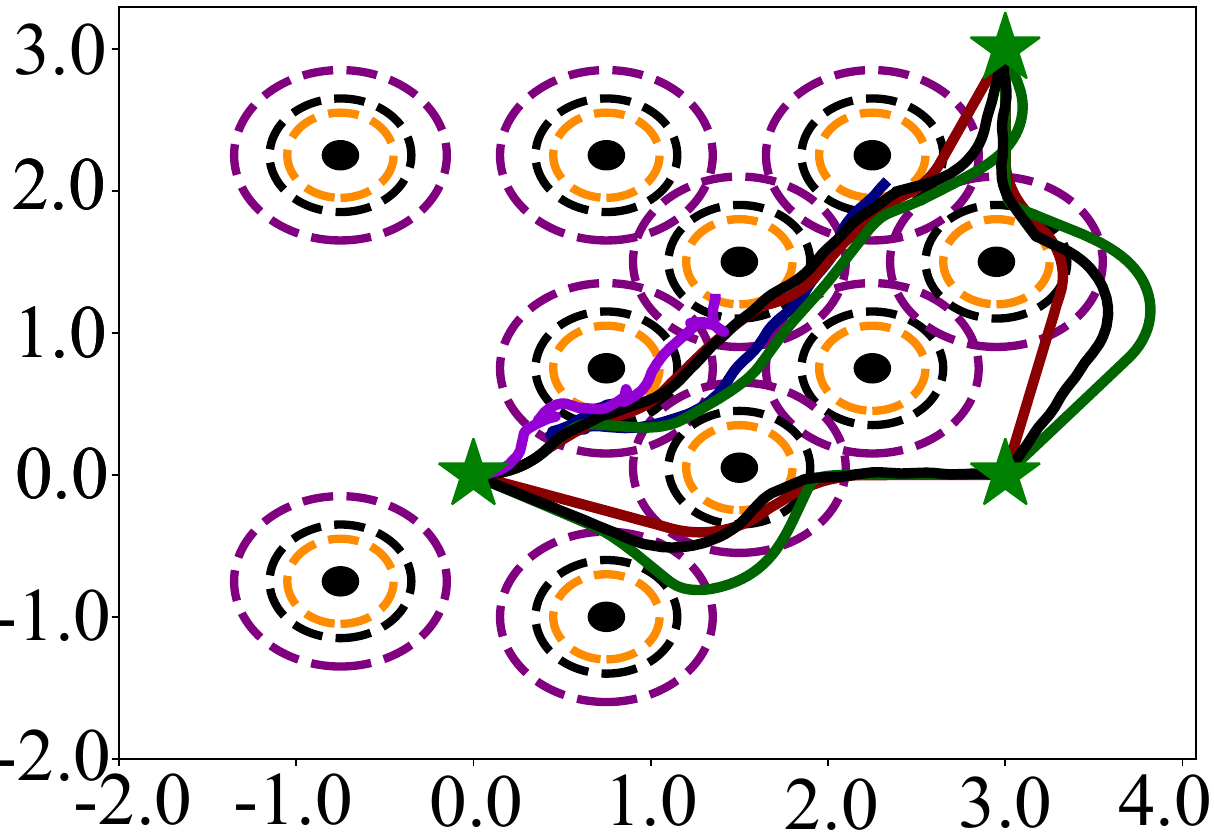}};
        \node at ($(fig63.south)+(0,-0.35)$) {\normalsize (c) ten obstacles};

        %\node[rectangle, rounded corners=1pt,minimum width=2.1cm,minimum height=0.3cm,draw=gray,fill=white!10](box) at ($(fig63.north east)+(-1.23,-0.36)$) {};
            \node[rotate=90] at ($(fig63.west)+(0,0)$) {\scriptsize y position [m]};
            \node at ($(fig63.south)+(-0.08,-0.02)$) {\scriptsize x position [m]};
    \end{scope}
    \node[inner sep=1pt, font=\footnotesize] at (0.36,-1.1) {$p_1$};
    \node[inner sep=1pt, font=\footnotesize] at (0.25,-2.65) {$p_2$};
    \node[inner sep=1pt, font=\footnotesize] at (-2,-2.3) {$p_3$};
    \begin{scope}[xshift=1cm, yshift=-3.8cm]
        \node[above right] (fig64) at (0,0){\includegraphics[ height=0.336\linewidth]{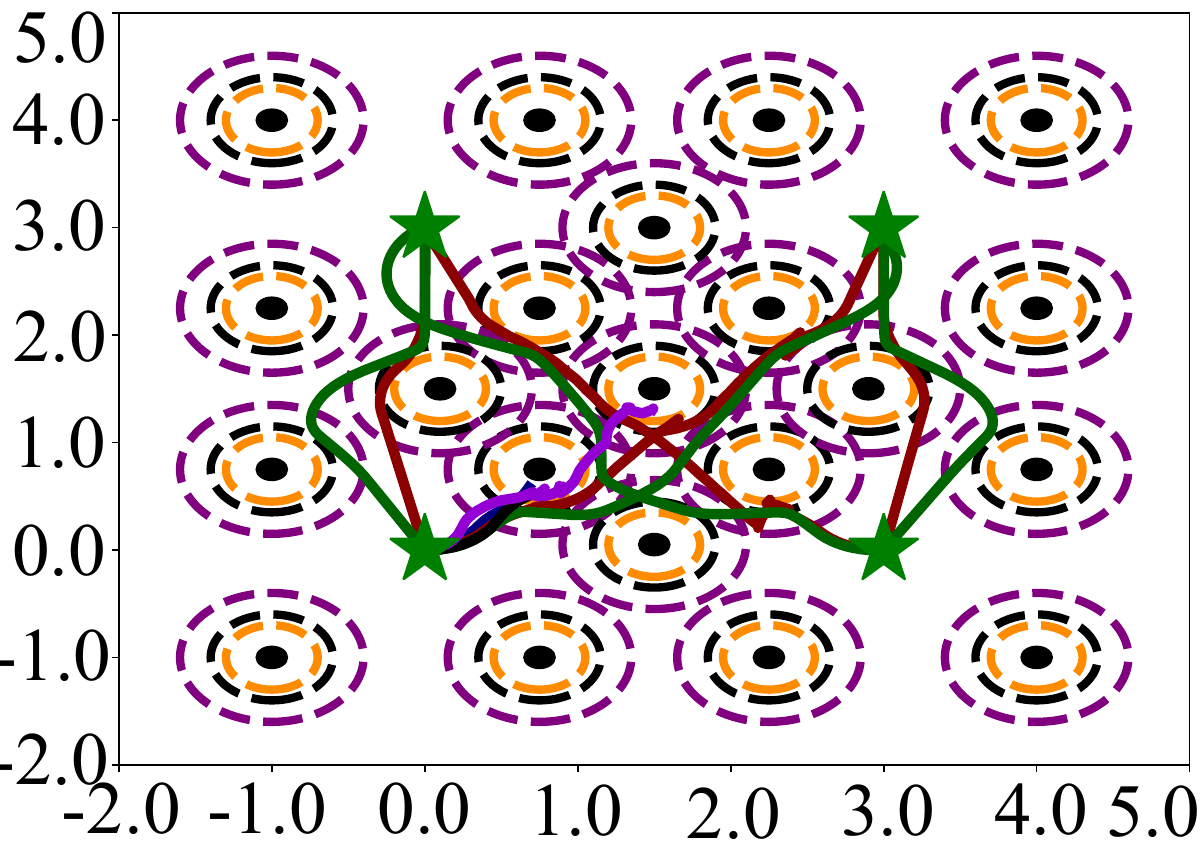}};
        \node at ($(fig64.south)+(0,-0.35)$) {\normalsize (d) twenty obstacles};
        %\node[rectangle, rounded corners=1pt,minimum width=2.1cm,minimum height=0.3cm,draw=gray,fill=white!10](box) at ($(fig64.north east)+(-1.23,-0.36)$) {};
            \node[rotate=90] at ($(fig64.west)+(-0.0,0)$) {\scriptsize y position [m]};
            \node at ($(fig64.south)+(0,-0.02)$) {\scriptsize x position [m]};
    \end{scope}
    \node[inner sep=1pt, font=\footnotesize] at (4.3,-1.4) {$p_1$};
    \node[inner sep=1pt, font=\footnotesize] at (4.3,-2.8) {$p_2$};
    \node[inner sep=1pt, font=\footnotesize] at (2.4,-1.42) {$p_3$};
    \node[inner sep=1pt, font=\footnotesize] at (2.4,-2.77) {$p_4$};
  \end{tikzpicture}
  \caption{Comparison of trajectories utilizing various methods under different numbers of static obstacles. The starting point is at the (0, 0) position, with p1, p2, p3, and p4 representing different goals in sequence.}
  \label{fig5}
  \vspace{-1em}
\end{figure}
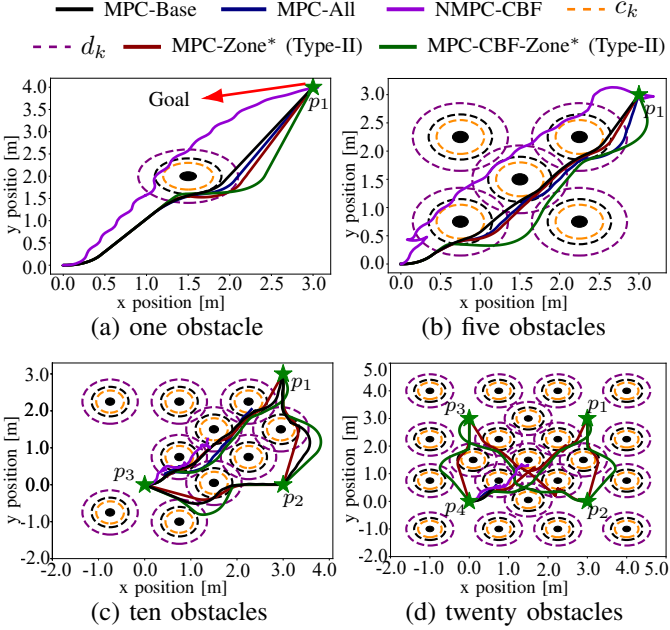
% 在表格外定义脚注内容
%\footnotetext{$^{*}$MPC-Zone and MPC-CBF-Zone are our proposed methods.}
reaches all four goals. Table \ref{tab: shape_control}  indicates that the MPC-CBF-Zone method maintains a 100Hz update frequency under any obstacle density. Conversely, despite a slower rate around 20Hz, MPC-Zone produces less conservative trajectories than MPC-CBF-Zone due to its multi-step forecasting. Finally, Table \ref{tab: change_accuracy_rising} confirms that our method can successfully reach each goal in sequence, regardless of the number of obstacles. All results collectively demonstrate the superiority of our method in scenarios with multiple static obstacles.

In dynamic obstacle environments, building on our method's proven efficacy in static contexts, we focus solely on the MPC-CBF-Zone technique. Fig. \ref{fig6} presents an analysis with ten obstacles, including spheres and cuboids, of which four are dynamic, showcasing our approach's local perception capability (Algorithm 1) to accurately differentiate between static and dynamic obstacles and effectively predict the trajectories of moving obstacles. The complete trajectory is detailed in Fig. \ref{fig7}(a), where the mobile robot achieves its goal in roughly 57s, maintaining an average optimization time of 0.014s. Concurrently, Fig. \ref{fig7}(b) confirms that both the linear velocity $v$ and angular velocity $w$ adhere to the prescribed constraints throughout the entire operation. These results affirm the efficacy of our approach in dynamic, multi-obstacle environments for obstacle recognition and navigation safety.
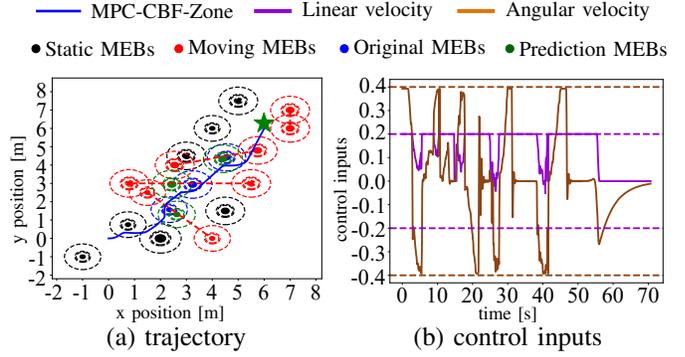
\begin{figure}[!t]
  \centering
    \begin{tikzpicture}
    % Legend - Line with black color
    \filldraw[black] (-10-0.9+0.8,-0.5) circle (1.5pt) node[right, text=black, font=\footnotesize]{Static MEBs};
    \filldraw[red] (-7.8+0.8-1.2,-0.5) circle (1.5pt) node[right, text=black, font=\footnotesize]{Moving MEBs};
    \filldraw[blue] (-5.6+0.8-1.2,-0.5) circle (1.5pt) node[right, text=black, font=\footnotesize]{Original MEBs};
    \filldraw[darkgreen] (-3.2+0.8-1.4,-0.5) circle (1.5pt) node[right, text=black, font=\footnotesize]{Prediction MEBs};
    % Legend - Line with dotted line
    \draw[blue, line width=1pt] (-10-0.9+0.8,0) -- (-9.5,0) node[right, text=black,font=\footnotesize]{MPC-CBF-Zone};
    % Legend - Line with darkred color
    \draw[ppurple, line width=1.5pt] (-6.8-0.5+0.1,0) -- (-6.3-0.5+0.1,0) node[right, text=black, font=\footnotesize]{Linear velocity};
    % Legend - Line with darkgreen color
    \draw[deeporange, line width=1.5pt] (-5.4-0.3+0.2+1,0) -- (-4.9-0.3+0.2+1,0) node[right, text=black, font=\footnotesize]{Angular velocity};
\end{tikzpicture}

  \hspace{-1.6em}
  \begin{tikzpicture}
    \begin{scope}[xshift=-3.5cm, yshift=0cm ]
        \node[above right] (fig63) at (0,0){\includegraphics[ height=0.351\linewidth]{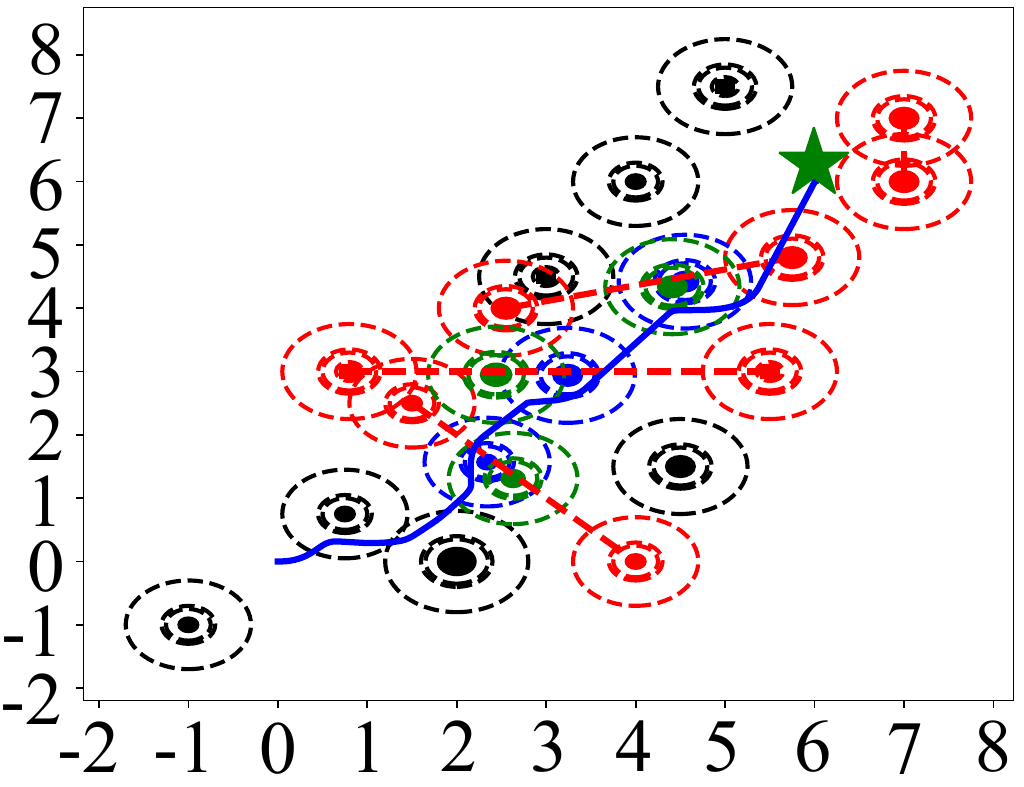}};
        \node at ($(fig63.south)+(0,-0.35)$) {\normalsize (a) trajectory};
        %\node[rectangle, rounded corners=1pt,minimum width=2.1cm,minimum height=0.3cm,draw=gray,fill=white!10](box) at ($(fig63.north east)+(-1.23,-0.36)$) {};
            \node[rotate=90] at ($(fig63.west)+(0,0)$) {\scriptsize y position [m]};
            \node at ($(fig63.south)+(-0.08,-0.02)$) {\scriptsize x position [m]};
    \end{scope}
    \begin{scope}[xshift=0.8cm, yshift=0cm]
        \node[above right] (fig64) at (0,0){\includegraphics[ height=0.351\linewidth]{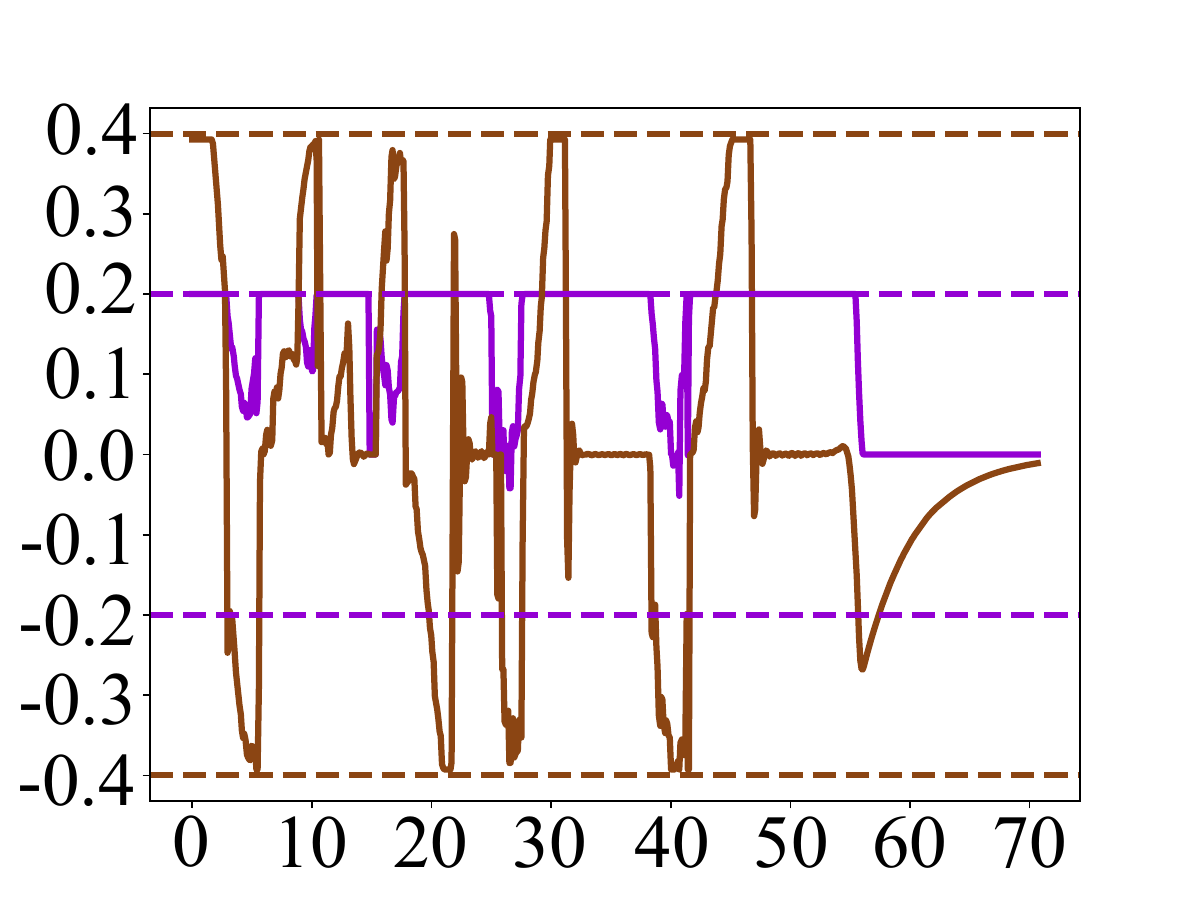}};
        \node at ($(fig64.south)+(0,-0.35)$) {\normalsize (b) control inputs};
        %\node[rectangle, rounded corners=1pt,minimum width=2.1cm,minimum height=0.3cm,draw=gray,fill=white!10](box) at ($(fig64.north east)+(-1.23,-0.36)$) {};
            \node[rotate=90] at ($(fig64.west)+(-0.0,0)$) {\scriptsize control inputs};
            \node at ($(fig64.south)+(0,-0.02)$) {\scriptsize time [s]};
    \end{scope}
  \end{tikzpicture}
  \caption{Trajectories utilizing MPC-CBF-Zone method and the control inputs of the mobile robot under moving obstacle scenarios.}
  \label{fig7}
  \vspace{-1em}
\end{figure}
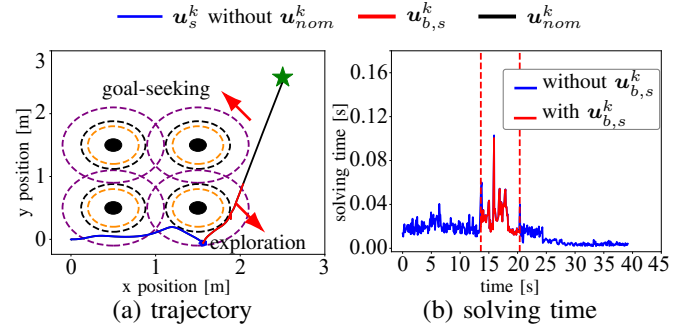
\begin{figure}[!t]
  \centering
    \begin{tikzpicture}
    % Legend - Line with black color
    \draw[blue, line width=1pt] (-10-0.9+0.8+1,0) -- (-9.5+1,0) node[right, text=black,font=\footnotesize]{$\boldsymbol{u}_{s}^{k}$ without $\boldsymbol{u}_{nom}^{k}$};
    % Legend - Line with darkred color
    \draw[red, line width=1.5pt] (-6.8-0.5+0.4+1,0) -- (-6.3-0.5+0.4+1,0) node[right, text=black, font=\footnotesize]{$\boldsymbol{u}_{b,s}^{k}$};
    % Legend - Line with darkgreen color
    \draw[black, line width=1.5pt] (-5.4-0.3+0.4+1,0) -- (-4.9-0.3+0.4+1,0) node[right, text=black, font=\footnotesize]{$\boldsymbol{u}_{nom}^{k}$};
\end{tikzpicture}

  \hspace{-1.08em}
  \begin{tikzpicture}
    \begin{scope}[xshift=-3.2cm, yshift=0cm ]
        \node[above right] (fig63) at (0,0){\includegraphics[ height=0.351\linewidth]{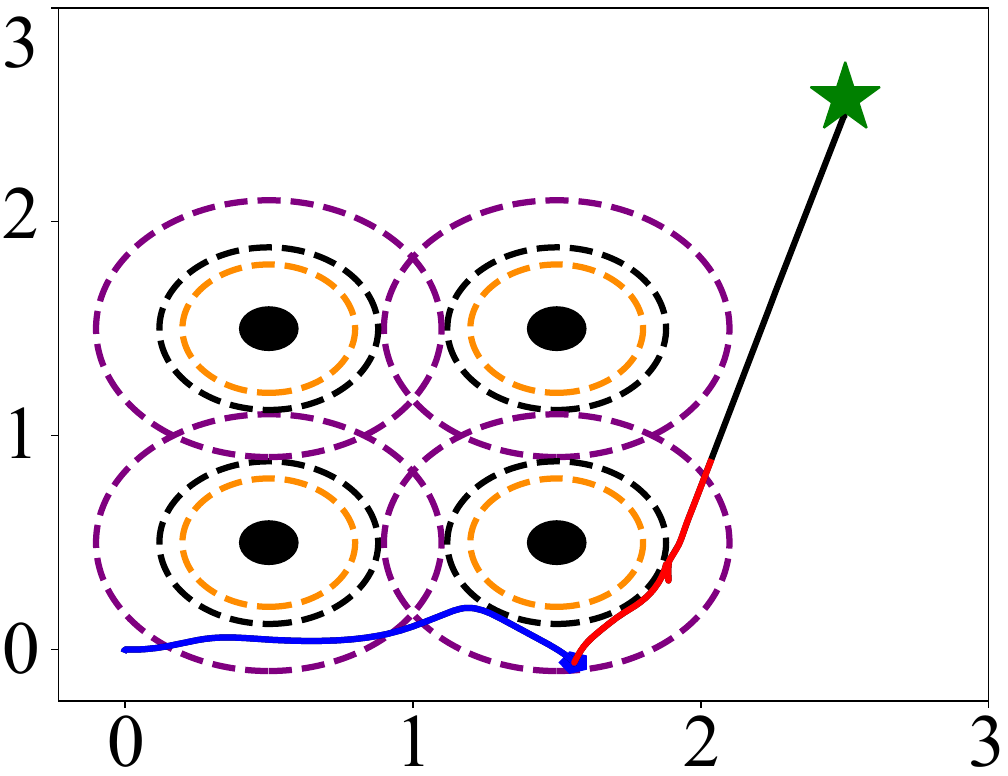}};
        \node at ($(fig63.south)+(0,-0.35)$) {\normalsize (a) trajectory};
        \draw[-{Latex[length=3mm, width=2mm]}, red, line width=1pt] (3.0,2.2) -- (2.6,2.6) 
         node[pos=1, left, text=black, font=\footnotesize]{goal-seeking};
         \draw[-{Latex[length=3mm, width=2mm]}, red, line width=1pt] (2.8,1.1) -- (3.2,0.7);
         \node[inner sep=1pt, font=\footnotesize] at (3.1,0.55) {exploration};
        %\node[rectangle, rounded corners=1pt,minimum width=2.1cm,minimum height=0.3cm,draw=gray,fill=white!10](box) at ($(fig63.north east)+(-1.23,-0.36)$) {};
            \node[rotate=90] at ($(fig63.west)+(-0.01,0)$) {\scriptsize y position [m]};
            \node at ($(fig63.south)+(-0.08,-0.02)$) {\scriptsize x position [m]};
    \end{scope}
    \begin{scope}[xshift=1.0cm, yshift=0cm]
        \node[above right] (fig64) at (0,0){\includegraphics[ height=0.351\linewidth]{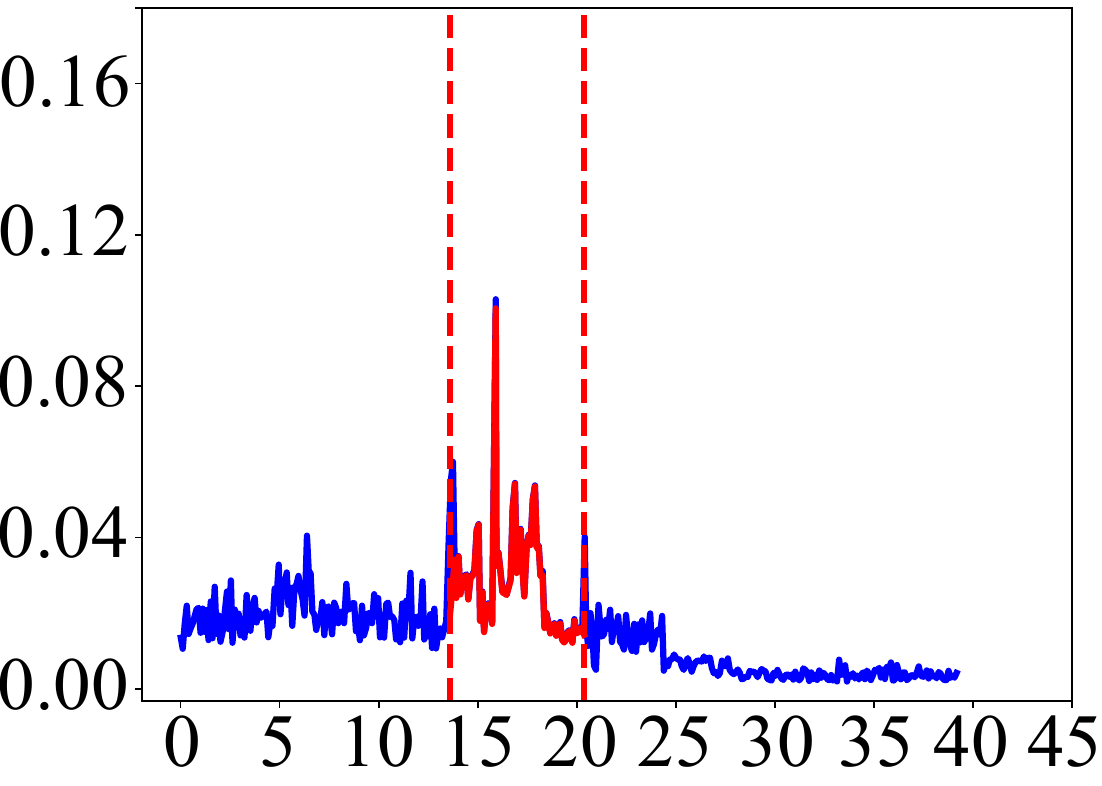}};
        \node at ($(fig64.south)+(0,-0.35)$) {\normalsize (b) solving time};
            \node[rectangle, rounded corners=1pt,minimum width=2.08cm,minimum height=0.8cm,draw=gray,fill=white!10](box) at ($(fig64.north east)+(-1.3,-0.8)$) {};
            \node (label) at ($(box.west)+(1.28,0.2)$) {\footnotesize without $\boldsymbol{u}_{b,s}^{k}$};
            \node (label2) at ($(box.east)+(-1.0,-0.2)$) {\footnotesize with $\boldsymbol{u}_{b,s}^{k}$};
            \draw[draw=blue,line width=1pt] ($(label.west)+(-0.25,-0.05)$) -- ($(label.west)+(0.05,-0.05)$) node (legend){};
            \draw[draw=red,line width=1pt] ($(label2.west)+(-0.25,-0.05)$) -- ($(label2.west)+(0.05,-0.05)$);
        %\node[rectangle, rounded corners=1pt,minimum width=2.1cm,minimum height=0.3cm,draw=gray,fill=white!10](box) at ($(fig64.north east)+(-1.23,-0.36)$) {};
            \node[rotate=90] at ($(fig64.west)+(-0.0,0)$) {\scriptsize solving time [s]};
            \node at ($(fig64.south)+(0,-0.02)$) {\scriptsize time [s]};
    \end{scope}
  \end{tikzpicture}
  \caption{Trajectory and solving time comparison for different controllers across various system modes.}
  \label{fig8}
  \vspace{-1em}
\end{figure}

Additionally, we validate the superiority of the backup controller $\boldsymbol{u}_{b,s}^{k}$. We consider extreme scenarios, specifically those where $\boldsymbol{u}_{nom}^{k}$ is not included in \eqref{all-controller}. Without considering $\boldsymbol{u}_{nom}^{k}$ as we do, the mobile robot would halt in the buffer zone, as indicated by the intersection of the blue and red lines in Fig. \ref{fig8}(a). However, when $\boldsymbol{u}_{b,s}^{k}$ is activated, the system automatically searches for a safe path as shown by the red line in Fig. \ref{fig8}(a). Ultimately, the system enters goal-seeking mode and reaches the goal within 40s. Although the backup controller $\boldsymbol{u}_{b,s}^{k}$ has the disadvantage of longer multi-step prediction solving times compared to $\boldsymbol{u}_{s}^{k}$, as shown in Fig. \ref{fig8}(b), our safety controller $\boldsymbol{u}_{s}^{k}$ incorporating $\boldsymbol{u}_{nom}^{k}$, ensures a low activation probability for $\boldsymbol{u}_{b,s}^{k}$. However, when activation is necessary, it results in smoother trajectories, as demonstrated in  Fig. \ref{fig5}.

\begin{figure*}[t]
  \centering
  \begin{minipage}[t]{0.48\textwidth}
  \begin{tikzpicture}
    \begin{scope}[xshift=-3.5cm, yshift=0cm ]
        \node[above right] (fig63) at (-2,0){\includegraphics[ height=0.62\linewidth]{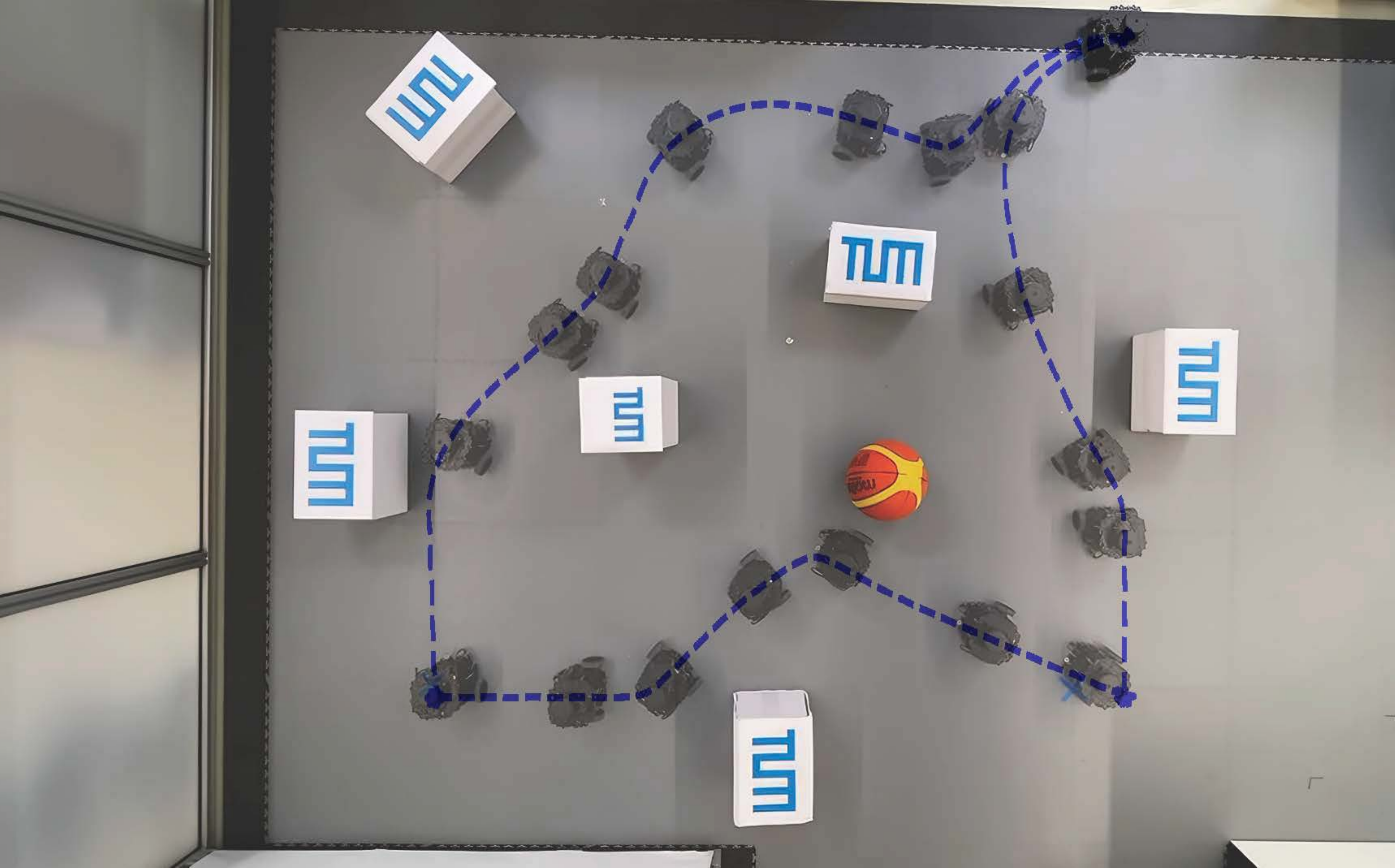}};
        %\node[rectangle, rounded corners=1pt,minimum width=2.1cm,minimum height=0.3cm,draw=gray,fill=white!10](box) at ($(fig63.north east)+(-1.23,-0.36)$) {};
    \end{scope}
    \draw[-{Latex[length=3mm, width=2mm]}, red, line width=1pt] (-1,2.2) -- (-2.6,3.3);
    \begin{scope}[xshift=1cm, yshift=0cm]
        \node[above right] (fig64) at (-6.5,3.2){\includegraphics[ height=0.25\linewidth]{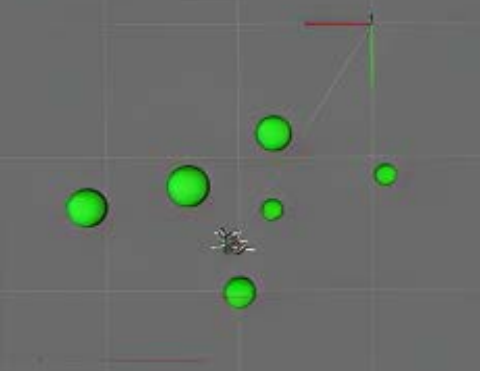}};
        \draw[red, dashed, line width=1pt] (-2,1.5) rectangle (-1.4,2.2);
        \draw[red, dashed, line width=1pt] (-5.2,3.9) rectangle (-4.8,4.2);
        \node[inner sep=1pt, font=\footnotesize, text=red] at (-4.9,3.5) {Mobile robot};
        \draw[green, dashed, line width=1pt] (-1+1.1,4.8) rectangle (-0.4+1.1,5.4);
        \node[inner sep=1pt, font=\footnotesize, text=green] at (-0.4+1.2,4.3) {Start point and $p_{3}$};
        \draw[green, dashed, line width=1pt] (-1+1.1,1) rectangle (-0.4+1.1,1.6);
        \node[inner sep=1pt, font=\footnotesize, text=green] at (-0.4+0.9,0.6) {Goal $p_{2}$};
        \draw[green, dashed, line width=1pt] (-1-2.9,1) rectangle (-0.4-2.9,1.6);
        \node[inner sep=1pt, font=\footnotesize, text=green] at (-0.4-3.2,0.6) {Goal $p_{1}$};
        %\node[rectangle, rounded corners=1pt,minimum width=2.1cm,minimum height=0.3cm,draw=gray,fill=white!10](box) at ($(fig64.north east)+(-1.23,-0.36)$) {};
    \end{scope}
  \end{tikzpicture}
  \caption{Experiment and perception result for static obstacle scenarios.}
  \label{fig9}
  \end{minipage}
    \hspace{.15in}
    \begin{minipage}[t]{0.48\textwidth}
        \begin{tikzpicture}
    \begin{scope}[xshift=-3.5cm, yshift=0cm ]
        \node[above right] (fig63) at (-2,0){\includegraphics[ height=0.62\linewidth]{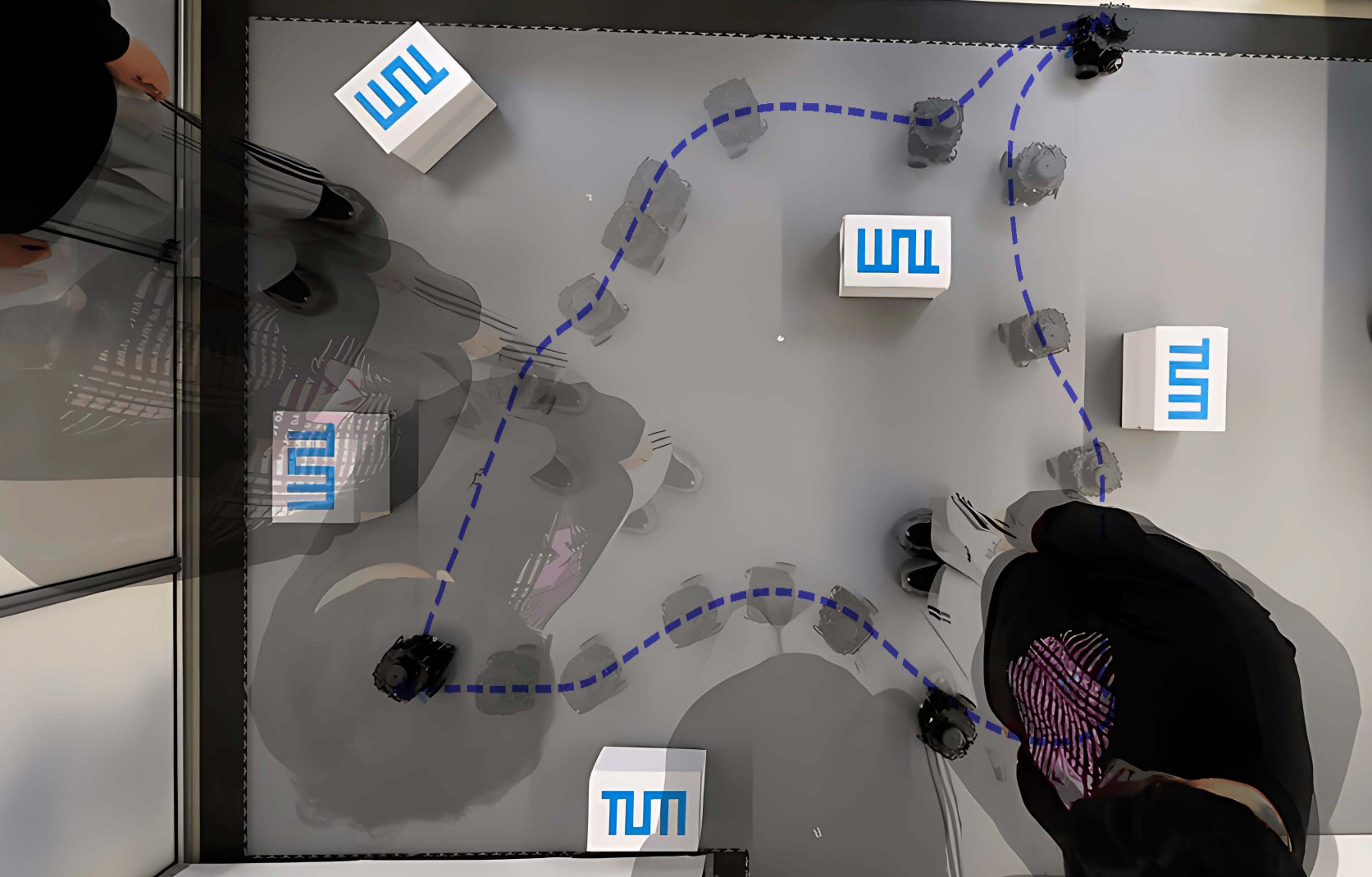}};
        %\node[rectangle, rounded corners=1pt,minimum width=2.1cm,minimum height=0.3cm,draw=gray,fill=white!10](box) at ($(fig63.north east)+(-1.23,-0.36)$) {};
    \end{scope}
    \draw[-{Latex[length=3mm, width=2mm]}, red, line width=1pt] (-2,3.9) -- (-2.5,4.3);
    \begin{scope}[xshift=1cm, yshift=0cm]
        \node[above right] (fig64) at (-6.5,3.2){\includegraphics[width=0.33\linewidth, height=0.245\linewidth]{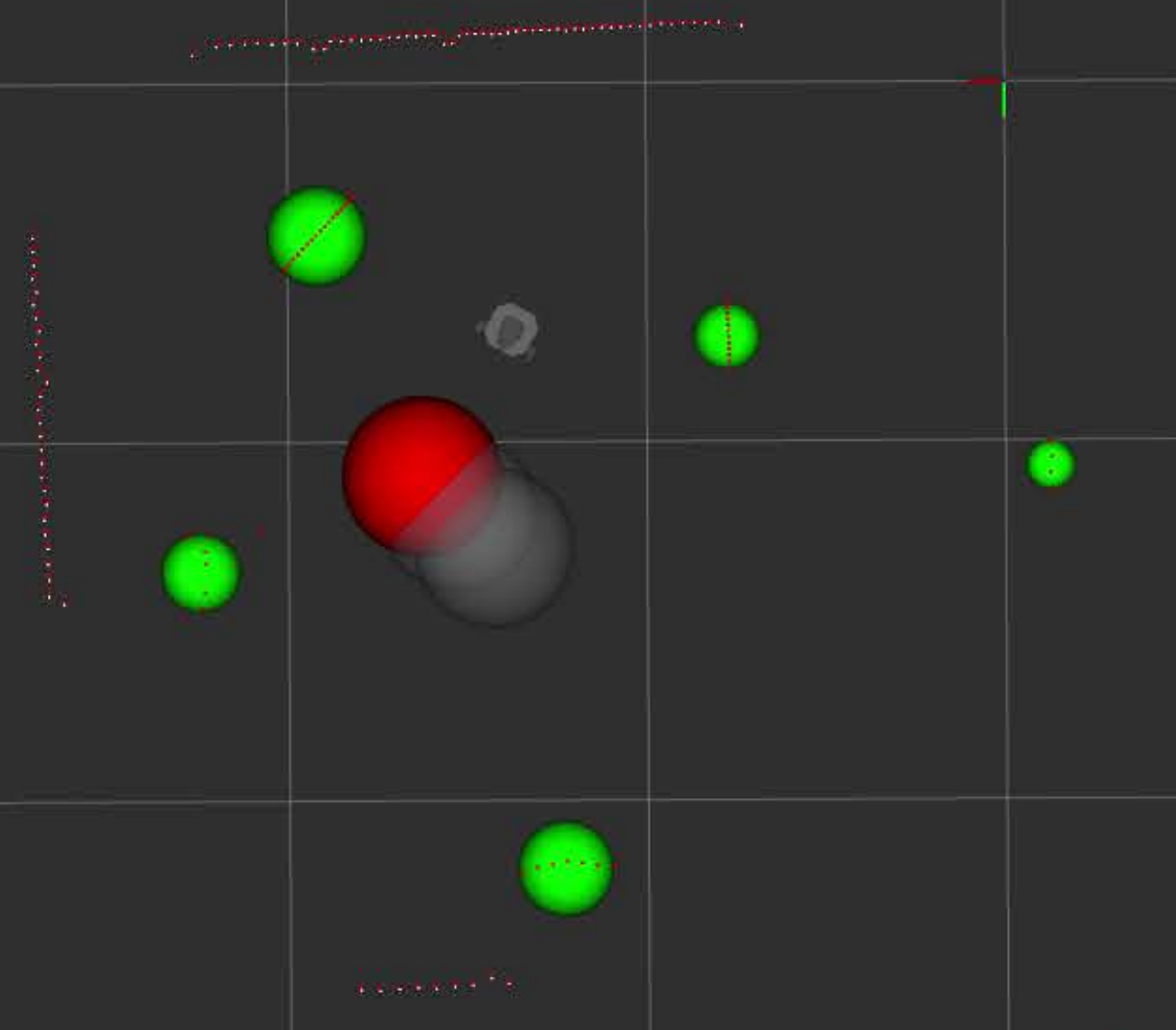}};
        \draw[red, dashed, line width=1pt] (-5.25,4.55) rectangle (-5.0,4.85);
        \node[inner sep=1pt, font=\footnotesize, text=red] at (-4.9,5.23) {Mobile robot};
        \draw[red, dashed, line width=1pt] (-2.95-0.1,3.25) rectangle (-2.35-0.1,3.85);
        \draw[green, dashed, line width=1pt] (-1+0.93,4.85) rectangle (-0.4+0.93,5.35);
        \node[inner sep=1pt, font=\footnotesize, text=green] at (-0.4+1.2,4.2) {Start point and $p_{3}$};
        \draw[green, dashed, line width=1pt] (-1+1.1,1.1) rectangle (-0.4+1.1,1.7);
        \node[inner sep=1pt, font=\footnotesize, text=green] at (-0.4+0.9,0.6) {Goal $p_{2}$};
        \draw[green, dashed, line width=1pt] (-1-3.1,1.1) rectangle (-0.4-3.1,1.7);
        \node[inner sep=1pt, font=\footnotesize, text=green] at (-0.4-3.2,0.6) {Goal $p_{1}$};
        %\node[rectangle, rounded corners=1pt,minimum width=2.1cm,minimum height=0.3cm,draw=gray,fill=white!10](box) at ($(fig64.north east)+(-1.23,-0.36)$) {};
    \end{scope}
  \end{tikzpicture}
  \caption{Experiment and perception result for dynamic obstacle scenarios.}
  \label{fig10}
  \end{minipage}
    
  \vspace{-0.1em}
\end{figure*}

Finally, the experimental results, illustrated in Figs. \ref{fig9}-\ref{fig10}, confirm the effectiveness of our approach in detecting both static and dynamic obstacles. 
\begin{table}[ht]
\caption{Comparison of time taken to reach the goal[Unit: \textcolor{red}{second}]}
\vspace{-0.5em}
\centering
\resizebox{0.48\textwidth}{!}{
\begin{tabular}{p{2.6cm}|c|c|c|c|c|c|c|c|c}
 \hline
 % \multicolumn{10}{|c|}{} \\
 % \hline
 & \multicolumn{1}{|c|}{1 obs} & \multicolumn{1}{|c|}{5 obs} & \multicolumn{3}{|c}{10 obs} & \multicolumn{4}{|c}{20 obs}\\
 \hline
   Goal     &$p_1$ &$p_1$ &$p_1$ &$p_2$ &$p_3$ &$p_1$ &$p_2$ &$p_3$ &$p_4$\\
 \hline
 MPC-Base &\textbf{26.9} &\textbf{22.5} &26.1  &60.8 &97.7 &-  & - & - & - \\
 MPC-All    &27.1 &25 &-  &- &- &-  &- &- &- \\
 NMPC-CBF    &32.8 &42.63 &-  &- &- &-  &- &- &- \\
 MPC-Zone$^{*}$   &29.2 &23.4  &\textbf{24.1}  &\textbf{58.7}  &\textbf{89.9}  &32.6  &71.2 &125.7 &169.3 \\
 MPC-CBF-Zone$^{*}$   &31.2 &29.3  &29.9  &70.0  &109.8  &\textbf{31.8}  &\textbf{68.2} &\textbf{116.3} &\textbf{159.2} \\
 \hline
\end{tabular}}\label{tab: change_accuracy_rising}
\vspace{-0.8em}
\end{table}
Notably, Fig. \ref{fig10} demonstrates its capability to accurately predict the trajectories of dynamic obstacles in environments among multiple obstacles, proving the real-world practical effectiveness of our perception algorithm. The strategic utilizing safety controller \eqref{all-controller} allows for safe, sequential goal navigation within the solving time of approximately 0.011s, reflecting high processing speeds. This not only highlights our method's proficiency in ensuring safety and operational efficiency in complex, dynamic settings but also demonstrates its exceptional ability to manage scenarios featuring numerous obstacles.

\section{Conclusions}

This paper presents a goal-seeking and exploration framework specifically designed to ensure safe control in dynamic environments featuring multiple obstacles. It enables the real-time identification and prediction of obstacles' states, guaranteeing that the state and quantity of obstacles do not affect the computation time required for control inputs. Moreover, it ensures safety while accomplishing predefined tasks. The effectiveness of our approach has been demonstrated through both simulation and practical hardware experiments, showcasing its potential for adaptation to more sophisticated, higher-order systems in the future.

\balance
\bibliographystyle{IEEEtran}
\bibliography{paper}

\end{document}